\documentclass[12pt]{article}

\usepackage{natbib}

\usepackage{etoolbox}
\newtoggle{colt}
\togglefalse{colt}

\usepackage{natbib}
\usepackage{boxedminipage}
\usepackage{multirow,nicefrac}
\usepackage{makecell,upgreek}
\usepackage{footnote}
\usepackage{longtable}
\usepackage{tablefootnote}
\usepackage[T1]{fontenc}
\usepackage{verbatim} 
\usepackage[utf8]{inputenc}

\usepackage{booktabs}

\iftoggle{colt}{}{
	\usepackage[margin=0.75in]{geometry}
\usepackage{subfigure}
\usepackage{algorithm2e}
}

\iftoggle{colt}{
	\usepackage{xcolor}    %
}{
\usepackage[dvipsnames]{xcolor}
}

\usepackage{colortbl}
\definecolor{lightgray}{gray}{0.9}  %

\iftoggle{colt}{}{
\usepackage{amsmath}
\usepackage{amsthm}

\usepackage[breaklinks=true]{hyperref}
\hypersetup{
	colorlinks=true,
	linkcolor=blue,
	citecolor=blue,
	urlcolor=blue
}
}

\iftoggle{colt}{
	
}{

\usepackage[capitalise,nameinlink]{cleveref}
\usepackage{amsfonts,amssymb,mathrsfs}
\usepackage{mathtools}

\usepackage{appendix}

\theoremstyle{plain}
	\newtheorem{theorem}{Theorem}

	\newtheorem{lemma}{Lemma}
	
	\newtheorem{corollary}{Corollary}
	\newtheorem{proposition}{Proposition}

	\theoremstyle{definition}

	\newtheorem{definition}{Definition}

}

\iftoggle{colt}{}{

}

\usepackage{autonum}

\newcommand{\norm}[1]{\left\|#1\right\|}
\newcommand{\abs}[1]{\left|#1\right|}
\newcommand{\inprod}[2]{\left\langle #1, #2 \right\rangle}

\newcommand{\rr}{\mathbb{R}}
\newcommand{\ee}{\mathbb{E}}
\newcommand{\pp}{\mathbb{P}}

\DeclareMathOperator*{\argmin}{argmin}

\newcommand{\poly}{\mathrm{poly}}

\DeclareMathOperator{\reg}{Reg}

\newcommand{\tstar}{t^\star}
\newcommand{\fstar}{f^\star}
\newcommand{\tstarp}{t^{\star'}}
\newcommand{\fstarp}{f^{\star'}}
\newcommand{\Lstar}{L^\star}

\newcommand{\pfref}[1]{Proof of \Cref{#1}}
\newcommand{\bigO}{\mathcal{O}}

\newcommand{\bigOtil}{\widetilde{\bigO}}

\newcommand{\wt}[1]{\widetilde{#1}}

\newcommand{\Omegatil}{\wt{\Omega}}

\newcommand{\Ktil}{\wt{K}}
\newcommand{\ttilstar}{\wt{t}^\star}
\newcommand{\ttilstarp}{\wt{t}^{\star'}}
\newcommand{\fhat}{\widehat{f}}
\newcommand{\yhat}{\widehat{y}}
\newcommand{\fat}{\mathrm{fat}}

\renewcommand{\epsilon}{\varepsilon}

\def\ddefloop#1{\ifx\ddefloop#1\else\ddef{#1}\expandafter\ddefloop\fi}
\def\ddef#1{\expandafter\def\csname bb#1\endcsname{\ensuremath{\mathbb{#1}}}}
\ddefloop ABCDEFGHIJKLMNOPQRSTUVWXYZ\ddefloop

\def\ddefloop#1{\ifx\ddefloop#1\else\ddef{#1}\expandafter\ddefloop\fi}
\def\ddef#1{\expandafter\def\csname frak#1\endcsname{\ensuremath{\mathfrak{#1}}}}
\ddefloop ABCDEFGHIJKLMNOPQRSTUVWXYZ\ddefloop

\def\ddefloop#1{\ifx\ddefloop#1\else\ddef{#1}\expandafter\ddefloop\fi}
\def\ddef#1{\expandafter\def\csname fr#1\endcsname{\ensuremath{\mathfrak{#1}}}}
\ddefloop ABCDEFGHIJKLMNOPQRSTUVWXYZ\ddefloop

\def\ddefloop#1{\ifx\ddefloop#1\else\ddef{#1}\expandafter\ddefloop\fi}
\def\ddef#1{\expandafter\def\csname eul#1\endcsname{\ensuremath{\EuScript{#1}}}}
\ddefloop ABCDEFGHIJKLMNOPQRSTUVWXYZ\ddefloop

\def\ddefloop#1{\ifx\ddefloop#1\else\ddef{#1}\expandafter\ddefloop\fi}
\def\ddef#1{\expandafter\def\csname scr#1\endcsname{\ensuremath{\mathscr{#1}}}}
\ddefloop ABCDEFGHIJKLMNOPQRSTUVWXYZ\ddefloop

\def\ddefloop#1{\ifx\ddefloop#1\else\ddef{#1}\expandafter\ddefloop\fi}
\def\ddef#1{\expandafter\def\csname b#1\endcsname{\ensuremath{\mathbf{#1}}}}
\ddefloop ABCDEGHIJKLMNOPQRSTUVWXYZ\ddefloop

\def\ddefloop#1{\ifx\ddefloop#1\else\ddef{#1}\expandafter\ddefloop\fi}
\def\ddef#1{\expandafter\def\csname bhat#1\endcsname{\ensuremath{\hat{\mathbf{#1}}}}}
\ddefloop ABCDEFGHIJKLMNOPQRSTUVWXYZ\ddefloop

\def\ddefloop#1{\ifx\ddefloop#1\else\ddef{#1}\expandafter\ddefloop\fi}
\def\ddef#1{\expandafter\def\csname btil#1\endcsname{\ensuremath{\tilde{\mathbf{#1}}}}}
\ddefloop ABCDEFGHIJKLMNOPQRSTUVWXYZ\ddefloop

\def\ddefloop#1{\ifx\ddefloop#1\else\ddef{#1}\expandafter\ddefloop\fi}
\def\ddef#1{\expandafter\def\csname bst#1\endcsname{\ensuremath{\mathbf{#1}^\star}}}
\ddefloop ABCDEFGHIJKLMNOPQRSTUVWXYZ\ddefloop

\def\ddefloop#1{\ifx\ddefloop#1\else\ddef{#1}\expandafter\ddefloop\fi}
\def\ddef#1{\expandafter\def\csname bst#1\endcsname{\ensuremath{\mathbf{#1}^\star}}}
\ddefloop abcdeghijklmnopqrstuvwxyz\ddefloop

\def\ddefloop#1{\ifx\ddefloop#1\else\ddef{#1}\expandafter\ddefloop\fi}
\def\ddef#1{\expandafter\def\csname bhat#1\endcsname{\ensuremath{\hat{\mathbf{#1}}}}}
\ddefloop abcdefghijklmnopqrstuvwxyz\ddefloop

\def\ddefloop#1{\ifx\ddefloop#1\else\ddef{#1}\expandafter\ddefloop\fi}
\def\ddef#1{\expandafter\def\csname b#1\endcsname{\ensuremath{\mathbf{#1}}}}
\ddefloop abcdeghijklnopqrstuvwxyz\ddefloop

\def\ddefloop#1{\ifx\ddefloop#1\else\ddef{#1}\expandafter\ddefloop\fi}
\def\ddef#1{\expandafter\def\csname barb#1\endcsname{\ensuremath{\bar{\mathbf{#1}}}}}
\ddefloop abcdefghijklmnopqrstuvwxyz\ddefloop

\def\ddef#1{\expandafter\def\csname c#1\endcsname{\ensuremath{\mathcal{#1}}}}
\ddefloop ABCDEFGHIJKLMNOPQRSTUVWXYZ\ddefloop
\def\ddef#1{\expandafter\def\csname h#1\endcsname{\ensuremath{\widehat{#1}}}}
\ddefloop ABCDEFGHIJKLMNOPQRSTUVWXYZ\ddefloop
\def\ddef#1{\expandafter\def\csname hc#1\endcsname{\ensuremath{\widehat{\mathcal{#1}}}}}
\ddefloop ABCDEFGHIJKLMNOPQRSTUVWXYZ\ddefloop
\def\ddef#1{\expandafter\def\csname t#1\endcsname{\ensuremath{\widetilde{#1}}}}
\ddefloop ABCDEFGHIJKLMNOPQRSTUVWXYZ\ddefloop
\def\ddef#1{\expandafter\def\csname tc#1\endcsname{\ensuremath{\widetilde{\mathcal{#1}}}}}
\ddefloop ABCDEFGHIJKLMNOPQRSTUVWXYZ\ddefloop

\usepackage[suppress]{color-edits}
\addauthor{ab}{red}
\addauthor{as}{cyan}

\title{Small Loss Bounds for Online Learning Separated Function Classes: A Gaussian Process Perspective}
\usepackage{times}

\author{Adam Block \\ Microsoft Research NYC \and Abhishek Shetty \\ MIT}

\date{}

\begin{document}

\maketitle

\begin{abstract}%
    
In order to develop practical and efficient algorithms while circumventing overly pessimistic computational lower bounds, recent work has been interested in developing \emph{oracle-efficient} algorithms in a variety of learning settings.  Two such settings of particular interest are online and differentially private learning.
While seemingly different, these two fields are fundamentally connected by the requirement that successful algorithms in each case satisfy stability guarantees; in particular, recent work has demonstrated that algorithms for online learning whose performance adapts to beneficial problem instances, attaining the so-called \emph{small-loss bounds}, require a form of stability similar to that of differential privacy.  In this work, we identify the crucial role that \emph{separation} plays in allowing oracle-efficient algorithms to achieve this strong stability.  Our notion, which we term $\rho$-separation, generalizes and unifies several previous approaches to enforcing this strong stability, including the existence of small-separator sets and the recent notion of $\gamma$-approximability.  We present an oracle-efficient algorithm that is capable of achieving small-loss bounds with improved rates in greater generality than previous work, as well as a variant for differentially private learning that attains optimal rates, again under our separation condition.  In so doing, we prove a new stability result for minimizers of a Gaussian process that strengthens and generalizes previous work. \looseness=-1

\end{abstract}

\section{Introduction}\label{sec:intro}

\abcomment{
    TODOs:
    \begin{enumerate}
        \item Move algs to top of page and put them in main results section.
        \item Write abstract.
        \item Polish body.
        \item Add related work
        \item @Abhishek Add example of separated class that we may care about (signals processing/tight frames/spectral condition/etc?)
    \end{enumerate}
}

In order to satisfy modern resource constraints it is important to develop computationally efficient algorithms for learning.  Unfortunately, many function classes of interest are computationally hard to learn, which prompts the search for algorithms that make efficient use of computational primitives, or oracle, that have heuristically motivated and practical implementation in settings of interest.  To wit, a common such primitive in learning is the notion of an Empirical Risk Minimization (ERM) Oracle, which grants one the ability to efficiently find a function that minimizes the empirical loss of some loss function on a given data set.  This oracle is motivated both by it sufficiency for classical learning \citep[Chapter 4]{shai} and its practical instantiation in a variety of settings of interest, such as deep neural networks \citep{lecun2015deep}.  
While a single call to an ERM oracle on independent data suffices for classical learning, there are many alternative settings where the extent to which this oracle can aid learning is less well understood; two examples of such are online and differentially private learning.

In online learning, a learner receives adversarially generated data points one in a series of $T$ rounds and, at each time step $t$, makes a prediction; the learner attempts to minimize his regret to the best function in a given class $\cF$ viewed in hindsight \citep{cesa2006prediction}.  
In addition to being an important learning setting in its own right, online learning algorithms are an important primitive in a number of other applications including sequential decision-making \citep{Lattimore_Szepesvári_2020}, algorithm design \citep{MWU}, learning in games \citep{cesa2006prediction}, and fairness \citep{haghtalab2023unifying}. 
While online learning is separated from classical learning by removing assumptions on the data, differentially private learning \citep{DworkMNS06} allows independent data but forces the algorithm to satisfy strong stability constraints that ensure that release of a trained model does not leak private information about the data on which it is trained.  Due to the increasing deployment of machine learning models, often trained on sensitive customer data, privacy has seen wide application in many fields \citep{dwork2014algorithmic} and has even been adopted by the US Census \citep{abowdcensus}.  While a priori these two settings may seem quite different, in fact they are intimately related because successful algorithms in both settings require some degree of \emph{stability}.  Indeed, a formal example of this connection is the fact that the combinatorial notion of complexity known as the \emph{Littlestone dimension} \citep{littlestone1988learning,ben2009agnostic} characterizes learnability in both settings \citep{AlonLMM19,bun2020equivalence}.

Another way in which online and differentially private learning are connected is in the study of small loss (or $\Lstar$) bounds in online learning \citep[Section 2.4]{cesa2006prediction}.
Typically, one expects regret for online learning to scale like $\bigO(\sqrt{T})$ in the worst case.  Often, however, one may expect $\cF$ to be chosen so that the $f \in \cF$ performing best in hindsight has total loss $\Lstar = o(T)$; in such cases, one may hope for algorithms  whose regret scales with $\Lstar$ instead of with $T$, automatically adapting to the beneficial problem instance.  Prior work \citep{hutter2004prediction,abernethy2019online,wang2022adaptive} has identified differential privacy as an essential ingredient to proving $\Lstar$ bounds for online learning for a family of algorithms conducive to the application of an ERM oracle, but the generality in which one can efficiently achieve such guarantees are unknown.  \looseness-1

In this work, we identify a new condition on function classes, $\rho$\emph{-separation} (\Cref{def:rho_separating}), that suffices to ensure oracle-efficient online and differentially private learning.  Our condition unifies and generalizes several earlier approaches including small separator sets \cite{dudik2020oracle,syrgkanis2016efficient} and $\gamma$-approximability \citep{wang2022adaptive}.  The notion of $\rho$-separation requires that functions in $\cF$ are sufficiently distinguishable in the sense that distinct functions are not too close to each other in $L^2(\mu)$ for some measure $\mu$; intuitively, this separation helps `boost' stability in an $L^2$ sense to differential privacy.  Using $\rho$-separation, we demonstrate the following:
\begin{itemize}
    \item We present an oracle-efficient algorithm for online learning (\Cref{thm:lstar_online}) that attains a small-loss ($\Lstar$) bounds whenever the function class is $\rho$-separated.  We compare our regret to that of alternative oracle-efficient algorithms and find that our approach generalizes and, in many settings, improves upon existing methods.
    \item We show that a variant of the algorithm used for online learning is capable of differentially private PAC-learning with optimal rates whenever $\cF$ is $\rho$-separated, improving on existing results in many cases.
\end{itemize}
In the course of our analysis, we prove a new stability bound for minimizers of Gaussian Processes (\Cref{lem:gp_stability_cor}) that stengthens prior such bounds both in the precise notion of stability controlled and the generality of the result itself.  In \Cref{sec:prelims} we state a number of prerequisite notions from learning theory as well as provide a formal definition of $\rho$-separation.  In \Cref{sec:main}, we provide rigorous statements of our main results as well as the oracle-efficient algorithms achieving the stated bounds; we conclude in \Cref{sec:gp_stability} by providing the main ideas behind our proofs.  We defer an extended discussion of related work to \Cref{app:related_work}.

\section{Prequisite Notions and $\rho$-Separation}\label{sec:prelims}

We now recall prerequisite notions, define the main problems of study, and introduce the fundamental assumption underlying our work: $\rho$-separation.

\subsection{Notions from Learning Theory}\label{ssec:prelims_learning}
Recall that the fundamental premise of learning theory is to assume the learner has access to some function class $\cF: \cX \to \cY$ mapping contexts $\cX$ to labels $\cY$, where we always consider $\cY \subset [-1,1]$ in this work.  Given a loss function $\ell: \cY \times \cY \to \rr_{\geq 0}$, the goal of the learner is to produce labels with minimal loss relative to the benchmark function class $\cF$.  A standard formal instantiation of this idea is the PAC-learning framework \cite{vapnik1964class,valiant1984theory}. 

\begin{definition}[PAC Learning]\label{def:pac}
    Given n independent and identically distributed data points $(X_i, Y_i)$ and a loss function $\ell$, an algorithm $\fhat$ depending on these data is said to $(\alpha, \beta)$-learn if for independent $(X, Y)$ it holds that $\pp\left( \ell(\fhat(X), Y) \geq \alpha + \inf_{f \in \cF} \ell(f(X), Y) \right) \leq \beta$.
    The number of samples required to $(\alpha, \beta)$-learn is referred to as the sample complexity.
\end{definition}

There are a wide variety of complexity measures of the function class, that dictate the sample  considered. 
In this work, we will work with the Gaussian complexity. 
\begin{definition}[Gaussian Complexity]
    Let \(\mathcal{F}\) be a class of functions from \(\mathcal{X}\) to \(\mathbb{R}\). For a fixed sample \(S = \{x_1, x_2, \ldots, x_n\} \subseteq \mathcal{X}\), the \emph{empirical Gaussian complexity} of \(\mathcal{F}\) is defined as \iftoggle{colt}{$\mathcal{G}_S(\mathcal{F}) = \mathbb{E}_{\xi}\left[ \sup_{f \in \mathcal{F}} \frac{1}{ \sqrt{n}} \sum_{i=1}^n \xi_i\, f(x_i) \right]$,}{
\begin{align}
    \mathcal{G}_S(\mathcal{F}) = \mathbb{E}_{\xi}\left[ \sup_{f \in \mathcal{F}} \frac{1}{ \sqrt{n}} \sum_{i=1}^n \xi_i\, f(x_i) \right],
\end{align}
    }
    where \(\gamma_1, \ldots, \gamma_n \sim \mathcal{N}(0,1)\) are independent standard Gaussians. 
\end{definition}
The Gaussian complexity of a function class is related to a number of natural notions of function class size.  
For example, if $\cF$ is finite, then $\cG_S(\cF) \lesssim \sqrt{\log(\abs{\cF})}$ \citep[Theorem 2.5]{boucheron2013concentration}. 
Furthermore, the Gaussian complexity is bounded by the square root of the \emph{VC dimension} \citep[Theorem 6]{bartlett2002rademacher}, which is the canonical measure of complexity in binary classification and is defined as follows \citep{shai}. 

\begin{definition}[VC Dimension] 
    Let $\mathcal{F}$ be a class of functions from a domain $\mathcal{X}$ to $\{0, 1\}$. 
    We say that $\mathcal{F}$ shatters a set $X \subseteq \mathcal{X}$ if for all $Y \subseteq X$, there exists a function $f \in \mathcal{F}$ such that $f(x) = 1$ if $x \in Y$ and $f(x) = 0$ otherwise.
    The VC dimension of $\mathcal{F}$ is the size of the largest set $X \subseteq \mathcal{X}$ such that there exists a function $f \in \mathcal{F}$ that shatters $X$.
\end{definition}

\paragraph{Online Learning.}
In contradistinction to PAC learning where data are iid and come all at once, in the online learning framework, a learner interacts with an adversarial environment over a sequence of rounds. At round \(t\), the adversary selects a context $x_t$ and a label $y_t$, while the leaner selects a hypothesis \(f_t\) from  some \(\mathcal{F}\); the context and label are then revealed to the learner, who suffers loss $\ell_t(f_t) = \ell(f(x_t), y_t)$, with the goal of minimizing \emph{regret} over $T$ rounds, defined to be
\begin{align}\label{eq:regret_def}
    \reg_T = \sum_{t=1}^T \ell_t(f_t) - \inf_{f \in \mathcal{F}} \sum_{t=1}^T \ell_t(f).
\end{align}
Rates in online learning have been classically established in a number of settings \citep{ben2009agnostic,rakhlin2011online,rakhlin2015sequential} and the minimax rate for simple function classes typically scales like $O(\sqrt{T})$ along with some complexity parameter of the class $\cF$.  Unfortunately, fully adversarial online learning is known to be strictly harder than PAC learning and its difficulty is characterized by the \emph{Littlestone dimension} \citep{littlestone1988learning,ben2009agnostic} as opposed to the Gaussian complexity or VC dimension above.   On the other hand, in some settings the minimax $\sqrt{T}$ regret scaling may be overly pessimistic.  For example, if the best hypothesis suffers a small loss $\Lstar = \inf_f \sum_{t = 1}^T \ell_t(f)$, e.g. in the realizable setting where $\Lstar=0$, we may hope to get algorithms whose regret scales only \emph{logarithmically} in $T$ and polynomially in $\Lstar$, which is a vast improvement over the minimax framework.  One of the main contributions of this work is to introduce an \emph{efficient} algorithm capable of achieving this rate whenever $\cF$ satisfies a separation condition discussed in the sequel.  We now discuss the other primary learning setting considered here.

\paragraph{Differential Privacy.}  While online learning makes success more difficult by removing assumptions on the data, in differential privacy, we place additional requirements on the algorithm.  
Motivated by data-privacy concerns, we consider the following notion of privacy \citep{DworkMNS06}. 
\begin{definition}[Differential Privacy]
    An algorithm $\mathcal{A}$ is said to satisfy $(\epsilon, \delta)$ (approximate) differential privacy  if for all datasets $S_1, S_2$ differing in a single element, and all measurable sets $A$ in the output space, we have
    \begin{align}
        \pp\left(\mathcal{A}(S_1) \in A\right) \leq e^\epsilon \pp\left(\mathcal{A}(S_2) \in A\right) + \delta.
    \end{align}
    If $\delta = 0$, we say that $\mathcal{A}$ satisfies $\epsilon$ (pure) differential privacy.
\end{definition}
The goal of private learning is to design an algorithm that is both an $(\alpha, \beta)$-learner and is $(\epsilon, \delta)$-differentially private with the number of samples scaling as $\poly\left( 1/\alpha, 1/\epsilon, \log(1/\beta), \log(1/\delta) \right)$. 
Recent work has demonstrated that differentially private learning is intimately connected with online learning in that the ability to succeed in either framework is characterized not by the Littlestone dimension mentioned above \citep{AlonLMM19} and $\Lstar$ bounds in online learning arise from private online algorithms \citep{hutter2004prediction,abernethy2019online,wang2022adaptive}.

\paragraph{Oracle Efficiency.}
    In both the online and private learning frameworks, due to practical resource constraints, we are interested in computationally efficient algorithms; unfortunately, for most cases of interest, learning is computationally intractable even for simple concept class such as halfspaces (see \citep{tiegel2022hardness} and references therein).
    To circumvent pessimistic fact and  develop theory that incorporate the success of practical heuristics, we consider the notion of oracle efficiency \citep{kalai2005efficient}.  In aprticular, we suppose that the learner has access to an \emph{Empirical Risk Minimization Oracle} for the class $\cF$ such that, given loss functions $\ell_1, \dots, \ell_k$ and data points $(x_1, y_1), \dots, (x_k, y_k)$ can efficiently return the minimizer of the cumulative loss:
    \begin{align}
        \fhat \in \argmin_{f \in \cF} \sum_{ i = 1}^k \ell_i(f(x_i), y_i).
    \end{align}
    The time complexity of a given algorithm is the number of oracle calls the algorithm makes multiplied by the number of inputs to each oracle call.  Our algorithms consider the special case where $k = n + m$, the $\ell_1 = \cdots = \ell_n = \ell$ and for $n < i \leq n + m$, $\ell_i(f(x_i), y_i) = \xi_i f(x_i)$ for $\xi_i$ a  gaussian. \looseness-1

\paragraph{Additional Notation.}  We denote by $\norm{\cdot}_m$ the empirical $L^2$ norm on $m$ points, assumed to be given and $\norm{\cdot}_1$, $\norm{\cdot}_2$ the (unnomralized) $\ell^1$ and Euclidean norms on $\rr^m$.  We use $f = \bigO(g)$ to say that $f/g$ is bounded and $f = \bigOtil(g) $ to mean that $f/g$ is bounded by polylogarithmic factors of the inputs to $f,g$; we use $f \lesssim g$ if $f = \bigO(g)$.

\subsection{Key Assumption: $\rho$-Separation}

    In this section, we will introduce the key assumption underlying our work: $\rho$-separation.

    \begin{definition}\label{def:rho_separating}
        Let $\cF: \cX \to \cY$ be a function class.  We call a measure $\mu$ on $\cX$ a $\rho$\emph{-separating measure} for $\cF$ if for all $f, g \in \cF$ such that $f \neq g$, we have $\norm{f-g}_{L^2(\mu)} \geq \rho$.  We abuse notation and say a set of points $\left\{ z_1, \dots, z_m \right\} \subset \cX$ is $\rho$-separating if the empirical measure is $\rho$-separating.
    \end{definition}
    In essence, $\rho$-separation ensures that no two functions in $\cF$ are too close together, which allows boosting stability in $L_2$ into differential privacy guarantees.  To gain intuition, note that $\rho$-separation is implied for finite function classes by a natural spectral property of the function class.
    \begin{proposition}\label{prop:singular_value}
        Let $\cF: \cX \to \rr$ be a finite function class, let $z_1, \dots, z_m \in \cX$ be points, and consider the matrix $A \in \rr^{m \times \abs{\cF}}$ whose columns are $(f(z_i))$.  Letting $\sigma$ be the minimal singular value of $A$, we have that $\cF$ is $\rho$-separated by the empirical measure on $\cX$ with $\rho = \sigma \sqrt{2/m}$.
    \end{proposition}
    We defer a proof of this result to \Cref{sec:separator_proofs}, which is a simple consequence of the definition of a singular value. 
    Such spectral properties of function classes are appear naturally in signals processing \citep{vetterli2014foundations}, wavelet theory \citep{stephane1999wavelet}, frame theory, and harmonic analysis \citep{benedetto2020harmonic}. 
    Learning with respect such function classes corresponds to the natural problem of recovering the large coefficient in a signal decomposition. 
    As a concrete example, if the function class corresponds to the Fourier basis (which it is easy to see satisfies our separation assumption), the learning problem corresponds to recovering the large Fourier coefficients of a signal. \looseness-1
    
    \ascomment{@Abhishek discuss with applications}

    Furthermore, $\rho$-separation can be seen as a generalization of the notion of a separator set \cite{goldman1993exact}, which has previously been used in the context of oracle efficient learning \citep{dudik2020oracle,DBLP:conf/focs/Neel0W19,block2024oracle}. 

    \begin{definition}[Separator Set]\label{def:separator_set}
        Let $\cF: \cX \to \left\{ \pm 1 \right\}$ be a function class.
        We call a set $S \subseteq \cX$ a separator set for $\cF$ if for all $f , g \in \cF$ such that $f \neq g$, there exists $x \in S$ such that $f(x) \neq g(x)$.
    \end{definition}
    In the special case of binary classification, if $S$ is a separator set then it yields a $\rho$-separating measure with $\rho = 1/\sqrt{\abs{S}}$.  Furthermore, in this special case, while the existence of a separating measure itself does not directly bound the complexity of a function class, under the additional assumption that the function class is learnable,
    we have that $\rho$-separation implies finiteness as well as .
    Further, we can show that the existence of $ \rho $-separating measure implies the existence of a separator set.
    \begin{proposition}\label{lem:separating_measure_sample}
        Let $\cF$ be a learnable, binary-valued function class and suppose there exists some measure $\mu$ that $\rho$-separates $\cF$.  Then $\cF$ is finite and has a separating set of size polynomial in $1/\rho$.
    \end{proposition}
    We defer a proof of this result to \Cref{sec:separator_proofs}, and note that it shows that, at least in the binary-valued case, $\rho$-separation and the existence of a small separator set are qualitatively equivalent, although we will see that quantitatively they can lead to different rates.  
    We emphasize that \Cref{lem:separating_measure_sample} implies that if $\cF$ is $\rho$-separated and PAC-learnable, then it is both online- and privately-learnable by finiteness in an information theoretic sense, without regard to computation time.
    The primary focus of our paper is demonstrating that this learning can be achieved (oracle-)efficiently.

\section{Main Results}\label{sec:main}

In this section, we present our main results. 
First, we will look at the small loss regret bounds for oracle efficient algorithms. 
We present nearly optimal regret bounds under the assumption that the function class is well-separated, unifying and improving on prior work.   Second, we will provide an oracle-efficient algorithm capable of achieving differentially private learning under separation that matches the sample complexity of a non-private learning algorithm, up to the standard $\epsilon^{-1}$ blow up due to privacy, improving upon the results of \cite{block2024oracle}.

\subsection{An Oracle-Efficient Algorithm with a Small-Loss Bound }\label{ssec:small-loss-bounds}

\iftoggle{colt}{
\begin{algorithm}[t]
    \caption{Follow the Perturbed Leader with Gaussian Perturbation}
    \label{alg:ftpl}
    \SetKwInput{KwInput}{Input}                %
    \SetKwInput{KwOutput}{Output}              %
    \DontPrintSemicolon
    \KwInput{Function class $\cF: \cX \to \cY$, horizon $T$, loss function $\ell$, $\rho$-separating set $\cZ = \left\{ Z_1, \dots, Z_m \right\}$, noise level $\eta > 0$, $L_0(f) = 0$}

    \For{$1 \leq t \leq T$}{
        Draw $\xi_1, \dots, \xi_m \sim \cN(0, 1)$ and let $\omega_t(f) = m^{-1/2} \cdot \sum_{i = 1} \xi_i f(Z_i)$. \\
        Set $f_t \gets \argmin_{f \in \cF} L_{t-1}(f) + \eta \cdot \omega_t(f)$. \\
        Observe $x_t$, play $\yhat_t = f_t(x_t)$, observe $y_t$ and incur $\ell(\yhat_t, y_t)$.
        Update $L_t(f) = L_{t-1}(f) + \ell(f(x_t), y_t)$ for all $f \in \cF$.
    }
\end{algorithm}
}
{
    \begin{algorithm}[t]
        \vspace{2pt} \hrule \vspace{4pt}  %
        \caption{Follow the Perturbed Leader with Gaussian Perturbation}
        \label{alg:ftpl}
        \SetKwInput{KwInput}{Input}                %
        \SetKwInput{KwResult}{Output}              %
        \DontPrintSemicolon

        \KwInput{Function class $\mathcal{F}: \mathcal{X} \to \mathcal{Y}$, 
                 horizon $T$, 
                 loss function $\ell$, 
                 $\rho$-separating set $\mathcal{Z} = \left\{ Z_1, \dots, Z_m \right\}$, 
                 noise level $\eta > 0$, 
                 initial loss $L_0(f) = 0$}
    
        \For{$1 \leq t \leq T$}{
            \tcp{Step 1: Sample Gaussian Perturbation}  
            Draw $\xi_1, \dots, \xi_m \sim \mathcal{N}(0, 1)$\;
            Compute perturbation: $\omega_t(f) = m^{-1/2} \sum_{i = 1}^{m} \xi_i f(Z_i)$\;
    
            \tcp{Step 2: Select Predictor}  
            Set $f_t \gets \arg\min_{f \in \mathcal{F}} L_{t-1}(f) + \eta \cdot \omega_t(f)$\;
    
            \tcp{Step 3: Play and Update Loss}  
            Observe $x_t$, play $\hat{y}_t = f_t(x_t)$, and observe $y_t$\;
            Incur loss $\ell(\hat{y}_t, y_t)$\;
            Update loss: $L_t(f) = L_{t-1}(f) + \ell(f(x_t), y_t)$ for all $f \in \mathcal{F}$\;
        }
    
        \vspace{4pt} \hrule \vspace{2pt}  %
    \end{algorithm}
    
}

We now present our first main result, a new, oracle-efficient algorithm for online learning capable of achieving a adaptive regret for bounded, separated function classes $\cF$.  Our algorithm is an instantiation of Follow the Perturbed Leader \citep{kalai2005efficient} and appears (with differently tuned parameters) in \citet{block2022smoothed,block2024oracle}.

\begin{theorem}\label{thm:lstar_online}
    Let $\cF: \cX \to [-1,1]$ be a function class and let $Z_1, \dots, Z_m \in \cX$ such that $\cF$ is $\rho$-separated by the empirical measure on the $Z_i$.  Let $\ell$ be a loss function bounded in $[-B,B]$ and suppose the learner plays according to \Cref{alg:ftpl}, i.e. at each time $t$, plays
    \begin{align}\label{eq:ftpl_gaussian}
        f_t \in \argmin L_{t-1}(f) + \eta \cdot \omega_t(f), \qquad \text{where} \qquad \omega_t(f) = \frac 1{\sqrt m} \cdot \sum_{i = 1}^{m} \xi_i f(Z_i),
    \end{align}
    with $\xi_i$ standard Gaussian random variables, $L_{t}$ the cumulative loss until time $t$, and $\eta > 0$ a regularization parameter.  For some choice of $\eta$, it holds that
    \begin{align}\label{eq:ftpl_gaussian_regret}
        \ee\left[ \reg_T \right] \leq 1 +   \frac{16 \sqrt{B \log\left( 2 T \abs{\cF} \right)\Lstar}}{\rho} + \frac{64B \log(2 T \abs{\cF})}{\rho^2}   .
    \end{align}
\end{theorem}
Note that in the regime where the separation $\rho$ is constant, an example of which we furnish in the sequel, we attain regret $\bigOtil\left( \sqrt{\log(\abs{\cF}) \cdot \Lstar} + \log(\abs{\cF})   \right)$, which is optimal up to logarithmic factors of $T$ \citep{cesa2006prediction}.  Furthermore,  we require only a single oracle call per round with a dataset of size $\bigO(T + m)$.  In the worst case, when no function in $\cF$ performs well in cumulative loss, $\Lstar \approx T$ and we recover a regret scaling as $\bigOtil\left( \sqrt{T \log(\abs{\cF})} / \rho \right)$, which is optimal for constant $\rho$; thus, if we can choose $m$ to be small, we achieve a new oracle-efficient regret bound.

As above, the downside of \Cref{thm:lstar_online} is that we require $\cF$ to be $\rho$-separated, which limits the generality of the bound, although we now highlight several key ways in which this separation can be achieved.  The first way is to construct new points $Z_1, \dots, Z_m$ that explicitly separate the function class $\cF$.  As we are assuming that $\cF$ is finite, we can simply introduce points $Z_1, \dots, Z_{\abs{\cF}}$ such that $f'(Z_f) = \bbI\left[ f = f' \right]$ for all $f, f' \in \cF$.  The downside of this approach is twofold: first, this requires modifying the assumption on the oracle, but, more importantly, it requres feeding in $\bigO(\abs{\cF})$ points to the oracle per round, precluding any hope of efficiency in the typical regime where we expect $\cF$ to be exponentially large.  This polynomial in $\abs{\cF}$ dependence is expected due to the computational lower bound of \citet{hazan2016computational}.  By significantly strengthening the oracle, we can avoid this issue, as we can back $\abs{\cF}$ points into a sphere of dimension $\log(\abs{\cF})$ with constant separation; while this solves the oracle-efficiency issues, it is not at all clear how to instantiate the oracle in practice, even heuristically.
A less general, but more optimistic, way to construct points $Z_1, \dots, Z_m$ separating $\cF$ is to assume sample access to some measure $\mu$ on $\cX$ that $\rho$-separates $\cF$.  In this case, standard uniform concentration bounds imply that we can take $m = \bigO\left( \log(\abs{\cF})  / \rho^2\right)$ (cf. \Cref{lem:separator_set_2}), which implies that achieving average regret $\epsilon$ is possible with $\poly\left( \log(\abs{\cF}) , \rho^{-1}, \epsilon^{-1}\right)$ time. \looseness=-1

A third instance in which $\rho$-separation holds is in the presence of a \emph{small separator set} (cf. \Cref{def:separator_set}) \citep{goldman1993exact,syrgkanis2016efficient}.  Because the existence of a separator set of size $m$ implies that $\cF$ is $\rho$-separated  with $\rho = m^{-1/2}$,  and $\Lstar \leq T$, we see that \eqref{eq:ftpl_gaussian} achieves expected regret $\bigOtil\left( m \log(\abs{\cF}) + \sqrt{m T \log (\abs{\cF})} \right)$, which shaves a $m^{1/4}$ factor from the dominant term in the bound of \citet[Theorem 2]{syrgkanis2016efficient}.  A related condition, termed $\delta$-admissability, is introduced in \citet{dudik2020oracle}, where they use this condition to achieve an oracle-efficient regret bound using a different variant of FTPL; we remark that their condition implies that $\cF$ is $\rho$-separated for a set of size $m$ with $\rho = \delta / \sqrt{m}$ and thus, with the same number of calls to the oracle with the same size set of inputs, our regret becomes (using the $\delta$ from \citet{dudik2020oracle}) $\bigOtil\left( m \log\left( \abs{\cF} \right) / \delta^2 + \sqrt{m T} / \delta \right)$, shaving a factor of $\sqrt{m}$ from that of \citet[Theorem 2.5]{dudik2020oracle}.  We emphasize that neither the result of \citet{syrgkanis2016efficient} nor that of \citet{dudik2020oracle} can adapt to the small loss setting where $\Lstar = o(T)$, but we see that even outside of this regime, \Cref{thm:lstar_online} improves upon the state of the art.

The key reason for the improvement of \Cref{thm:lstar_online} is that the Gaussian perturbation we choose in \eqref{eq:ftpl_gaussian} is naturally adapted to $\ell_2$ geometry as opposed to the $\ell_1$ geometry to which the Laplacian perturbations of many alternative instantiations of FTPL are adapted.  Thus, unlike these other perturbations, there is no explicity cost in \emph{regret} that \Cref{thm:lstar_online} pays for adding more points $Z_i$; the only reason to not increase $m$ is \emph{computational}, i.e., we do not wish to increase $m$ solely in order to avoid the additional time required to feed more points into the oracle at each round.

\paragraph{Comparison to \citet{wang2022adaptive}.} 
While the aforementioned results are an interesting consequence of \Cref{thm:lstar_online}, the purpose of the algorithm is to consider regimes where $\Lstar$ is small and we can see improved regret.  Such bounds for FTPL were first introduced in \citet{hutter2004prediction} for general classes, but the efficiency of their algorithm scales linearly in $\abs{\cF}$, which is not acceptable in the standard setting where $\cF$ is exponentially large.  We thus compare \Cref{thm:lstar_online} to the approach of \citet{wang2022adaptive}, which presents an efficient instantiation of FTPL capable of achieving $\Lstar$ bounds whenever there exists what  a $\gamma$-\emph{approximable} matrix for $\cF$, defined as follows.
\begin{definition}[Definition 2 from \citet{wang2022adaptive}]\label{def:approximable}
    Given a function class $\cF: \cX \to [-1,1]$, a matrix $\Gamma \in \rr^{\abs{\cF} \times m}$ is said to be $\gamma$-\emph{approximable} with respect to $\cF$ if for all $f \in \cF$ and $(x,y) \in \cX \times \cY$, there is some $s \in \rr^m$ such that $\norm{s}_1 \leq \gamma$ and for all $f' \in \cF$, it holds that $\inprod{\Gamma_f - \Gamma_{f'} }{s} \geq \ell(f(x), y) - \ell(f'(x), y)$.
\end{definition}
Using this definition, \citet[Theorem 1]{wang2022adaptive} shows that given a $\gamma$-approximable matrix and an ERM oracle, one can achieve
\begin{align}\label{eq:wang_regret}
    \ee\left[ \reg_T \right] \lesssim 1 + \left( \log(\abs{\cF}) + \sqrt{m \log(\abs{\cF})} + \gamma \right) \cdot \sqrt{\Lstar} + \gamma^2 + \gamma \left( \log(\abs{\cF}) + \sqrt{m \log\left( \abs {\cF} \right)} \right),
\end{align}
with $T$ calls to the oracle and each call having an input size of $\bigO\left( T + m \right)$.  As the oracle complexity of this algorithm is the same as that of \Cref{thm:lstar_online} over $T$ rounds, we must only compare the resulting regret, i.e., \eqref{eq:ftpl_gaussian_regret} to \eqref{eq:wang_regret}.  
A key point to note is that even as $\gamma \to 0 $, the above regret bound has as its dominant term $ \sqrt{ \Lstar m  \log \abs{ \mathcal{F} }  } $ which can be large even for natural classes with $ \rho = O(1) $.   

\begin{proposition} \label{prop:hadamard_gap}
    For any $m$, there exists a class  $\cF$ with $\abs{ \mathcal{F} } = 2^m$ on a domain of size $m$ with $\rho$-separating measure $\mu$ such that \Cref{alg:ftpl} has regret $\bigOtil\left( \sqrt{ \Lstar \cdot m } \right)$, while the algorithm of \citet{wang2022adaptive} has regret $\bigOtil\left( m \sqrt{ \Lstar } \right)$ with the same runtime.
    
    In addition, there is a class $\mathcal{F}$ with $\abs{ \mathcal{F} }$ on a domain of size $2^m$ such that \Cref{alg:ftpl} has regret $\bigOtil\left( \sqrt{ \Lstar m } \right)$,  the algorithm of \citet{wang2022adaptive} has regret $\bigOtil\left(  \sqrt{ \Lstar m 2^m } \right)$ with the same runtime.
\end{proposition}
    The classes presented here are modifications of the Fourier transform matrix. 
    The reason to present these two results separately (instead of just the exponential gap) is to present the gap in the common regime where the function class size is much larger than that of the domain. This result shows that the algorithm of \citet{wang2022adaptive} suffers because even as $\gamma$ tends to 0, the regret bound still depends on $m$. 
    On the other hand, we note that through duality, $\gamma$-approximability implies $\rho$-separation, although at the cost of an additional factor of $\sqrt{m}$.
We formally present this in the following lemma and defer the proof to \Cref{sec:approximability}. 

\begin{corollary}\label{cor:gamma_approx}
    Suppose that $\cF$ has a $\gamma$-approximable matrix $\Gamma \in \rr^{\abs{\cF} \times m}$.  Then running the Gaussian FTPL algorithm in \eqref{eq:ftpl_gaussian} achieves regret
    \begin{align}
        \ee\left[ \reg_T \right] \lesssim 1 + \gamma \cdot \sqrt{m \log\left( T \abs{\cF} \right) \Lstar} + m \gamma^2 \log\left( T \abs{\cF} \right).
    \end{align}
\end{corollary}

Comparing \Cref{cor:gamma_approx} to \eqref{eq:wang_regret}, we see that on the dominant term, up to logarithmic factors, we have replaced a term scaling like $\bigOtil\left( \sqrt{\Lstar} \cdot \left( \gamma + \sqrt{m \log\left( \abs{\cF} \right)} \right) \right)$ with one scaling like $\bigO\left( \gamma \cdot \sqrt{m \log\left( \abs{\cF} \right) \Lstar} \right)$, which amounts to a $\sqrt{m}$ reduction in efficiency.

Summarizing the above discussion, we see that while $\gamma$-approximability implies $\rho$-separation at the cost of a polynomial blowup in the regret, there are natural examples of function classes that have large separation of which the algorithm of \citet{wang2022adaptive} does not fully take advantage, leading to large (even exponential) suboptimality with respect to \Cref{alg:ftpl}.

\subsection{Privacy}\label{ssec:privacy}

    In this section, we show that an adaptation of \Cref{alg:ftpl}, whose pseudocode can be found in \Cref{alg:erm_perturbed}, is an oracle-efficient, differentially private learner.  Indeed, the following result shows that under constant separation, our algorithm achieves optimal rates in these condtions.
\iftoggle{colt}{
    \begin{algorithm2e}[t]
        \hrulefill
        \caption{Perturbed ERM with Gaussian Perturbation}
        \label{alg:erm_perturbed}
        \SetKwInput{KwInput}{Input}                %
        \SetKwInput{KwOutput}{Output}              %
        \DontPrintSemicolon
        \KwInput{Dataset $S = \{ (x_1, y_1), (x_2,y_2) \dots, (x_n,y_n) \}$, separating distribution $\mu$, auxiliary sample size $m$, loss function $\ell(w,x)$, hypothesis class $\mathcal{F}$, noise level $\eta$}
        Draw $m$ independent samples $Z = \{z_1, z_2, \dots, z_m\} \sim \mu^m$\;
        Draw $m$ independent samples $\omega = \{ \omega_1, \omega_2, \dots, \omega_m\} \sim \mathcal{N} ( 0 , I) $\;
        Compute     $\hat{f} \gets \arg\min_{f \in \mathcal{F}}  \sum_{i=1}^{n} \ell(f(x_i),y_i) + \frac{ \eta }{ \sqrt{m}} \sum_{i=1}^{m} \omega_i f(z_i)$ \;
        \Return $\hat{f}$ \\
        \hrulefill
    \end{algorithm2e}
}{
    \begin{algorithm}[t]
        \vspace{2pt} \hrule \vspace{4pt}  %
        \caption{Perturbed ERM with Gaussian Perturbation}
        \label{alg:erm_perturbed}
        \SetKwInput{KwInput}{Input}                %
        \SetKwInput{KwResult}{Output}              %
        \DontPrintSemicolon
    
        \KwInput{Dataset $S = \{ (x_1, y_1), \dots, (x_n,y_n) \}$, 
                 separating distribution $\mu$, 
                 auxiliary sample size $m$, 
                 loss function $\ell(w,x)$, 
                 hypothesis class $\mathcal{F}$, 
                 noise level $\eta$}
    
        \KwResult{Trained function $\hat{f}$}
    
        \tcp{Step 1: Draw Samples}  
        Draw $m$ independent samples $Z = \{z_1, z_2, \dots, z_m\} \sim \mu^m$\;
    
        \tcp{Step 2: Sample Gaussian Perturbation}  
        Draw $m$ independent samples $\omega = \{ \omega_1, \dots, \omega_m\} \sim \mathcal{N} ( 0 , I)$\;
    
        \tcp{Step 3: Compute Perturbed ERM}  
        Compute $\hat{f} \gets \arg\min_{f \in \mathcal{F}}  
        \sum_{i=1}^{n} \ell(f(x_i),y_i) + \frac{\eta}{\sqrt{m}} \sum_{i=1}^{m} \omega_i f(z_i)$\;
    
        \Return $\hat{f}$\;
        \vspace{4pt} \hrule \vspace{2pt}  %
    \end{algorithm}
}

    \begin{theorem}
        \label{thm:main_privacy}
        Let $\cF: \cX \to [-1,1]$ be a function class and $\ell$ be a Lipschitz loss function bounded by 1. 
        Further, assume that the points $Z_1, \dots, Z_m \in \cX$ are a $\rho$-separating set for $\cF$.  Then \Cref{alg:erm_perturbed} is $(\epsilon, \delta)$-differentially private if $\eta \gtrsim \nicefrac{\left( \cG(\cF) + \sqrt{\log(1/\delta)} \right)}{\rho^2\epsilon}$.
        Furthermore, for $\epsilon \leq 1$, \Cref{alg:erm_perturbed} $(\alpha, \beta)$-learns $\cF$ if 
        \begin{align}
            n \geq \max\left\{  \frac{ \mathcal{G} ( \mathcal{F}   )^2 + \log(1/\beta) }{  \alpha^2} ,  \frac{ \mathcal{G} ( \mathcal{F}   )^2 + \sqrt{ \log(1/\delta) \log(1/\beta) } }{ \alpha \epsilon \rho^2  }    \right\}
        \end{align}

    \end{theorem}
    As we described in \Cref{sec:prelims}, for $\cF$ with VC dimension at most $d$, it holds that $\cG(\cF) \lesssim \sqrt{d} $ and thus the above guarantee recovers the optimal sample complexity of learning \citep{shai} with the typical $\epsilon^{-1}$ blow up due to the privacy requirement.
    Thus, in the presence of the separating condition, we achieve, with a differentially private algorithm, the optimal sample complexity of a \emph{non-private} algorithm, in an oracle-efficient manner.
    Note that such a guarantee is not possible without the separation condition since the sample complexity of private learning is bounded in terms of the Littlestone dimension of the class \citep{AlonLMM19}.

    The work of \citet{Neel0VW20} initiated the study of oracle-efficient differentially private algorithms with a focus on binary-valued function classes $\cF$ with separating sets of size $m$.  Their algorithm, related to our \Cref{alg:erm_perturbed}, achieves a sample complexity bound of $n \gtrsim \nicefrac{d}{\alpha^2} + \nicefrac{m^{3/2}}{\alpha \epsilon}$, where $d$ is the VC dimension of $\cF$.  
    For direct comparison, we state a corollary of our result 

    \begin{corollary}
        Suppose that $\cF$ has a separator set of size $m$.  Then, \Cref{alg:erm_perturbed} achieves $(\epsilon, \delta)$-differential privacy with sample complexity $ n \geq \max\left\{ \nicefrac{d}{\alpha^2} , \nicefrac{m d}{\alpha^2 \epsilon} \right\}  $.
    \end{corollary}
    
    Though\footnote{We remark that \citet{Neel0VW20} also presents an algorithm that is capable of achieving \emph{pure} differential privacy with $\delta = 0$ with a dependence of $m^2$ instead of $m^{3/2}$.}  these guarantees are in general incomparable as at worst $d \leq m$, typically we have $d \ll m$ which can be lead to an improvement in sample complexity.   On the other hand, examples such as \Cref{prop:hadamard_gap} demonstrate that in favorable cases, \Cref{alg:erm_perturbed} can have sample complexity independent of $m$.  

    We also compare our results to those of \citet{block2023oracle}, who studied the related setting of oracle-efficient private learning with access to unlabelled public data (from a closely related distribution).
    They study a variant of \Cref{alg:erm_perturbed} (with an additional output perturbation step) and show that it is $(\epsilon, \delta)$-differentially private for general $\cF$, but the sample complexity that is achieve are large polynomial functions of the accuracy parameter $\alpha$ and the privacy parameters $\epsilon$ and thus significantly worse than the rates appearing in \Cref{thm:main_privacy}.   In addition to the quantitative improvements, while less general, we note that the $\rho$ separation condition does \emph{not} have to do with the test distribution at all and thus certifying the separation condition can be seen as a preprocessing step that can amortized for many future learning tasks with the same hypothesis class.

\section{Analysis Techniques}\label{sec:gp_stability}
In this section we sketch the proofs of our main results above.  We begin with the technical result underlying our main theorems, which establishes a stability property for Gaussian processes.  Then, in \Cref{ssec:online-learning-proofs}, we prove \Cref{thm:lstar_online} using this stability result.  

\subsection{An Improved Gaussian Process Stability Bound}\label{ssec:gp-stability}
We now describe the key Gaussian process stability bound on which our results rely.  This bound is a stronger form of the bounds found in \citet{block2022smoothed,block2024oracle}, which conclude that approximate minimizers of a Gaussian process are, in some sense, close to the true minimizer with reasonable probability. We now formally define a Gaussian process. 
\begin{definition}[Gaussian Process]
    Let $T$ be an index set. 
    A collection of random variables $ \{\omega(t) \}_{t \in T} $ is a Gaussian process if for all finite subsets $T' \subseteq T$, the random vector $\{ \omega(t)\}_{t \in T'}$ has a multivariate Gaussian distribution.
    The function $ m : T \to \mathbb{R} $ given by $ m(t) = \ee [\omega(t)] $ is called the mean function of the Gaussian process and 
    the function $ K : T \times T \to \mathbb{R} $ given by $ K(t, t') = \ee [(\omega(t) - m(t))(\omega(t') - m(t'))] $ is called the covariance kernel of the Gaussian process.
    
\end{definition}
We recall that a Gaussian process on an index set $T$ can be fully characterized by its mean function $m : T \to \rr$ and covariance kernel $K: T \times T \to \rr$, and define a Gaussian process now.  The covariance kernel in particular can be interpreted as a measure of distance between points in $T$, with two points $s,t \in T$ being considered close if $K(s, t) \approx K(t,t)$. 
In particular, define the distance between two points $s,t \in T$ as $\norm{s-t} = \sqrt{K(s,s) + K(t,t) - 2K(s,t)}$.
We use the norm notation though $T$ need not be a vector space; this is standard notation and can be justified using standard representation theorems for Gaussian processes.  

The results of \citet{block2022smoothed,block2024oracle} involve \emph{approximate minimizers} of a Gaussian process $\Omega: T \to \rr$ with mean function $m$ and covariance kernel $\eta^2 \cdot K$ for some variance parameter $\eta > 0$, where $t \in T$ is an approximate minimizer if $\Omega(t) \leq \Omega(\tstar) + \tau$ for some tolerance $\tau > 0$ for $\tstar = \argmin \Omega(s)$.  In particular, those results say that assuming the covariance kernel is well conditioned in the sense that $\sup_t K(t,t) / \inf_t K(t,t)$ is not too large, then
\begin{align}\label{eq:gp_stability_old}
    \pp\left( \exists t \in T \text{ s.t. } \Omega(t) \leq \Omega(\tstar) + \tau \text{ and } \nicefrac{K(s,\tstar)}{K(\tstar, \tstar)} \leq 1 - \rho^2 \right) 
    \lesssim \frac{\tau \cdot \ee\left[ \sup_{t \in T} \Omega(t) - m(t) \right]}{\rho^2 \eta^2}.
\end{align}
The upper bound on the probability that there exist approximate minimizers of $\Omega$ far from the true minimizer scales naturally with the tolerance $\tau$ and distance $\rho$ as well as the size of $T$, as measured by the Gaussian complexity, suggesting that the larger an index set is, the more likely it is that far away approximate minimizers exist.  Moreover, as the scale of the noise, parameterized by $\eta$, increases, the probability of far away approximate minimizers decreases, reflecting the fact that the true minimizer becomes less dependent on the mean function $m$.

For learning theory applications, we can instantiate \eqref{eq:gp_stability_old} in the case where $T = \cF$ is a function class, $\omega$ is a centred Gaussian process with covariance kernel $K(f, f') = \inprod{f}{f'}_m$, the empirical inner product on a dataset $\left\{ Z_1, \dots, Z_m \right\}$, and $\Omega(f) = m(f) + \eta \cdot \omega(f)$ for some mean function $m$.  Given an alternative mean function $m'$ such that $\sup_f \abs{m(f) - m'(f)} \leq \tau$, we can define the offset process $\Omega'(f) = m'(f) + \eta \cdot \omega(f)$ and let $\fstarp = \argmin \Omega'(f)$.  Applying \eqref{eq:gp_stability_old} and some rearranging then implies that with reasonable probability,
\begin{align}\label{eq:gp_stability_old_cor}
    \|\fstar - \fstarp\|_m^2 \lesssim \frac{\tau \cdot \ee\left[ \sup_f \omega(f)\right]}{\eta},
\end{align}
i.e., minimizers of Gaussian processes with similar mean functions are close in $\norm{\cdot}_m$.  While powerful, unfortunately \eqref{eq:gp_stability_old} and its corollary \eqref{eq:gp_stability_old_cor} do not yield sufficiently strong guarantees to establish small loss bounds for online learning, nor does it immediately yield a differentially private algorithm.  Instead, as explained in the sequel, a stronger guarantee is necessary to establish these results. \looseness-1
\begin{lemma}\label{lem:gp_stability_cor}
    Let $\cF: \cX \to \rr$ be a function class and $\mu$ a measure on $\cF$ with $\omega: \cF \to \rr$ the canonical Gaussian process associated to this space; further assume that $\mu$ is a $\rho$-separating measure on $\cF$ in the sense that $\norm{f -f '}_{L^2(\mu)} \geq \rho$ for all $f \neq f' \in \cF$.  
    Let $m,m': \cF \to \rr$ be two mean functions satisfying $\sup_{f \in \cF} \abs{m(f) - m'(f)} \leq \tau$.  If $\fstar = \argmin_f \Omega(f)$, where $\Omega(f) = m(f) + \eta \cdot \omega(f)$ and $\fstarp, \Omega'$ are defined similarly then as long as 
    \begin{align}\label{eq:gp_dp_bound_eta}
        \eta \geq \frac{32\tau}{ \rho^2} \left( \ee\left[ \sup_{f \in \cF} \omega(f) \right] + \sqrt{2\log\left( \frac 2\delta \right)} \right),
    \end{align}
    it holds for any $f$ that
    \begin{align}\label{eq:gp_dp_bound}
        \pp\left( \fstar = f \right) \leq \left( 1 + \frac{32 \tau}{\eta \rho^2} \left( \ee\left[ \sup_{f \in \cF} \omega(f) \right] + \sqrt{2 \log\left( \frac 2\delta \right)} \right) \right) \cdot \pp\left( \fstarp = f \right) + \delta.
    \end{align}
\end{lemma}
In contradistinction to \eqref{eq:gp_stability_old}, \Cref{lem:gp_stability_cor} requires an additional assumption of \emph{separation} on the function class $\cF$, which cannot be dropped in general because its application to private learning, which is only possible with finite Littlestone dimension \citep{AlonLMM19}.  What is gained from this reduction in generality, however, is the fact that the guarantee on stability is significantly stronger than that of \eqref{eq:gp_stability_old_cor} in that it provides fine-grained control on the \emph{distributional} stability of the approximate minimizers, which is essential for our applications.  Furthermore, we eliminate the assumption of well-conditionedness on the covariance kernel, which substantially increases the generality of application of this result, in particular by allowing functions with arbitrarily small norm.

While we defer a full proof of \Cref{lem:gp_stability_cor} to \Cref{app:gp_stability}, we sketch the argument here.  The proof begins in the same way as that of \citet{block2022smoothed,block2024oracle}, by introducing the `bad' event $\cE(\rho, \tau)$ corresponding to the existence of a $\tau$-approximate minimizer existing at lesat distance $\rho$ from the true minimizer.  In this first step we additionally assume that the covariance kernel $K$ is well-conditioned in the sense that $1 \geq K(f,f) \geq \kappa^2 > 0$ for some $\kappa$, a requrement we later drop.  As in those earlier bounds, we then fix an arbitrary $y \in \rr$ and condition on the event that $\fstar = t$ and $\Omega(\fstar) = y$; unlike those earlier works, however, we also require conditioning on the event $\Phi_\delta$, occuring with probability at least $1 - \delta$, that the supremum of $\omega$ is not too much larger than its expectation, the reason for which we explain below.  As $\Phi_\delta$ occurs with high probability, we may condition on this event without losing more than a constant in the final bound; similarly, we will ignore the `good' event $\cE^c(\rho, \tau)$, as the resulting probability of this event will appear on the right hand side of \eqref{eq:gp_dp_bound}.  The proof then uses a fundamental fact of Gaussian Processes: the covariance kernel of a pinned Gaussian Process does not depend on the value at which this process is pinned.  After careful analysis using several more properties specific to Gaussian Processes, we conclude
\begin{align}
    \pp\left( \fstar = f \text{ and } \cE(\rho, \tau) \text{ and } \Phi_\delta \right) \lesssim \frac{\tau}{\eta \kappa^2 \rho^2} \cdot \ee\left[ \omega(\fstar) \bbI[\Phi_\delta] | \tstar = t \right] \cdot \pp\left( \tstar = t \right).
\end{align}
It is here that relying on $\Phi_\delta$ is essential, as under $\Phi_\delta$, it holds that $\omega(\fstar)$ is not so large, even conditioned on the possibly low probability event that $\fstar = f$.  A stronger version of \eqref{eq:gp_stability_old} follows after rearranging terms.  In order to remove the dependence on $\kappa$, we introduce an auxiliary process $\Omegatil$ by adding additional points $\wt{Z}_i$ such that $f(\wt{Z}_i) = 1$ for all $f$.  Adding enough of these points ensures that the auxiliary process is well-conditioned and careful comparisons of $\Omegatil$ to $\Omega$ as well as a sharp understanding of the relationship between the induced covariance kernel and the norm $\norm{\cdot}$ on $\cF$ allows us to conclude the desired result.  Note that the introduction of this auxiliary process is a subtle point as one \emph{cannot use this approach to remove the $\kappa$ parameter from } \eqref{eq:gp_stability_old} because the resulting separation between points as measured by correlation is hurt by the indroduction of these auxiliary points; in order for this approach to work, we carefully relate the norm induced on $\cF$ by $\Omegatil$ to that by $\Omega$, which is significantly better behaved than the correlation.  With the main technical step thus proved, we now apply this result to online learning.

\subsection{Online Learning Proofs}\label{ssec:online-learning-proofs}
While a full proof of \Cref{thm:lstar_online} can be found in \Cref{app:online_learning}, we sketch the main argument here and present the full algorithm. 
The proof proceeds by applying \Cref{lem:gp_stability_cor} to the well-known `Be-the-Leader' lemma from \citet{kalai2005efficient,cesa2006prediction}, which states that `Follow-the-Leader' style algorithms have regret that can be bounded by the stability of the predictions.  In particular, the lemma concludes that if $f_t$ is chosen as in \eqref{eq:ftpl_gaussian}, then
\begin{align}\label{eq:btl_body}
    \ee\left[ \reg_T \right] \leq 2 \eta \cdot  \ee\left[ \sup_{f \in \cF} \omega(f) \right] + \sum_{t = 1}^{T} \ee\left[ \ell_t(f_t) - \ell_t(f_{t+1}) \right].
\end{align}
Here, the first term measures the `bias,' i.e., how large the perturbation is, while the second term measures the stability of the predictions.  For Lipschitz $f$, it suffices to control the stability of predictions measured by the $L^2$ norm of the perturbation, as was done for the case of smoothed data in \citet{block2022smoothed}.  In order to realize $\Lstar$ bounds, however, \citet{hutter2004prediction} used a stronger notion of stability akin to differential privacy, which was then further explored in \citet{wang2022adaptive}.  This stronger notion of stability is precisely what is controlled by \Cref{lem:gp_stability_cor}; thus, we apply that result 
with $\delta = \nicefrac{1}{T \cdot \abs{\cF}}$ and, bounding the Gaussian complexity by $\sqrt{\log(\abs{\cF})}$ \citep{boucheron2013concentration}, we sum over $f \in \cF$ and see that for sufficiently large $\eta$,
\begin{align}
    \ee\left[ \ell_t(f_t) - \ell_t(f_{t+1}) \right] \lesssim  \frac{1}{T} + \frac{B \sqrt{\log(T \abs{\cF})} }{\eta \rho^2} \cdot \ee\left[ \ell_t(f_{t+1}) \right].
\end{align}
Plugging the preceding display into \eqref{eq:btl_body}, we observe that the regret is bounded by
\begin{align}
    \ee\left[ \reg_T \right] \lesssim 1 + \eta \cdot \sqrt{\log(\abs{\cF})} +  \frac{B \sqrt{\log(T \abs{\cF})} }{\eta \rho^2} \cdot \sum_{t= 1}^T \ee\left[ \ell_t(f_{t+1}) \right].
\end{align}
The result concludes by controlling the summation in the last term by $\Lstar$, which is another consequence of the Be-the-Leader Lemma, and balancing $\eta$ subject to \eqref{eq:gp_dp_bound_eta}.

\iftoggle{colt}{}{

\iftoggle{colt}{
\section{Proof of \Cref{thm:main_privacy}}
}{
    \subsection{Differential Privacy Proofs}
}
\label{sec:proof_main_privacy}

    In this section, we apply that Gaussian stability bound in order to improve oracle efficient private algorithms. 
    Recall that the \Cref{alg:erm_perturbed} outputs a hypothesis $\hat{f}$ that is a minimizer of the empirical risk with a Gaussian perturbation term added. 
    Thus, the output of the algorithm can be seen as corresponding to the minimizer of a Gaussian process with the mean function given by the loss on the input data and since we are using a set of auxiliary points drawn from a separating distribution, it follows that the conditions of \Cref{lem:gp_stability_cor} are satisfied.
    This instantiated in the present context gives us the following privacy guarantee.
    
    \begin{lemma}[Privacy of Perturbed ERM]\label{lem:privacy_perturbed}
        Let $\delta, \eta > 0$.  
        Suppose that $\mu$ is a $\rho$ separating measure for $\mathcal{F}$.
        The output of the \Cref{alg:erm_perturbed} is $(\epsilon, \delta)$-differentially private for 
        \begin{align}
              \epsilon \;=\; \frac{8}{\eta\, \rho^2} \left(  \mathcal{G}( \mathcal{F})  + \sqrt{2\log\Bigl(\frac{2}{\delta}\Bigr)} \right).
        \end{align} 
    \end{lemma}
    
    \begin{proof}
        In order to prove the privacy of the algorithm, we need to verify that the conditions of \Cref{lem:gp_stability_cor}.
    Fix the data set $S = \left\{ (x_1 , y_1) , \dots , (x_n , y_n) \right\} $ 
    Let the Gaussian process corresponding to the hypothesis class $\mathcal{F} $  be $\Omega(f) =  \sum_i \ell(f(x_i) , y_i ) +    \frac{1}{ \sqrt{m}} \sum_{i=1}^{m} \xi_i f(z_i)$.
    Note that the mean function $m_S(f) =  \sum_i \ell(f(x_i) , y_i )$ and the covariance kernel $K(f,f') = \frac{1}{m} \sum_{i=1}^{m} f(z_i) f'(z_i)$.

    Note that we assumed that $\mu$ is a $\rho$ separating measure for $\mathcal{F}$.
    From \Cref{lem:separating_measure_sample}, we know that a sample of size $m = (d + \log(100/ \delta)) \rho^{-2}$ also induces a $ \rho / 2 $ separating measure with probability at least $1 - \delta/100$. 
    This implies that the Gaussian process $\omega(f)$ is $\rho/2$ separated with respect to the norm induced by the bilinear form $K$.
    When the dataset is changed from \(S\) to a neighboring dataset \(S'\), the change in the mean function \(m_S(f)\) is bounded by
    \[
    \sup_{f\in\mathcal{F}  } \bigl| m_S(f) - m_{S'}(f) \bigr| \le 1.
    \]
    We apply this to \Cref{lem:gp_stability_cor} with $ \tau = 1 $, and setting 
    \[
    \eta =  \frac{8} {\rho^2 \epsilon } \left( \mathbb{E}\Bigl[\sup_{f\in\mathcal{F}} \omega(f)\Bigr] + \sqrt{\log\Bigl(\frac{2}{\delta}\Bigr)} \right),
    \]
    then for any \(f\in \mathcal{F}\) the selection probabilities satisfy
    \[
    \mathbb{P}\Bigl( \hat{f}(S) = f \Bigr) \;\le\; \left(1 + \frac{8\tau}{\eta\,  \rho^2} \left( \mathbb{E}\Bigl[\sup_{f\in\mathcal{F}} \omega(f)\Bigr] + \sqrt{2\log\Bigl(\frac{2}{\delta}\Bigr)} \right) \right) \mathbb{P}\Bigl( \hat{f}(S') = f \Bigr) + \delta.
    \]

    Simplifying the above expression using the identity $1+ x \leq e^x $, we get for any measurable set \(O\) that 
    \[
    \mathbb{P}\Bigl( \hat{f}(S) \in O \Bigr) \;\le\; e^{\epsilon} \, \mathbb{P}\Bigl( \hat{f}(S') \in O \Bigr) + \delta,
    \]
   as required. 
    \end{proof}

    It remains to be shown that the output of the algorithm \cref{alg:erm_perturbed} is still accurate on the test distribution. 
    In order to reason about this, the key quantity to control is the size of the perturbation term. 
    Fortunately, since we are using Gaussian perturbations, this is naturally bounded in terms of the noise level $\eta$ and the Gaussian complexity of the function class which in turn bounds the generalization error of the ERM procedure even if the Gaussian perturbation were not present.

    \begin{lemma}[Accuracy of Perturbed ERM]\label{lem:accuracy_perturbed}
        Let $\beta, \eta > 0$.  
        The output of the \Cref{alg:erm_perturbed} satisfies $(\alpha, \beta)$ accuracy as long as the number of samples $n$ satisfies  
        \begin{align}
            n \geq \max\left\{ \frac{\mathcal{G}( \mathcal{F} )^2 + \log(1/\beta)}{\alpha^2}  , \eta  \frac{\mathcal{G}( \mathcal{F} ) + \sqrt{ \log(1/\beta) }}{\alpha}    \right\} .
        \end{align}  
    \end{lemma}
    \begin{proof}
        Note that from standard uniform convergence bound in terms of Gaussian complexity, we have that with probability at least $1 - \beta $, that for all $f \in \mathcal{F}$,
        \begin{align}
            \left| \frac{1}{n} \sum_{i=1}^{n} \ell(f(x_i) , y_i ) - \mathbb{E} \ell(f(x) , y ) \right| \leq \alpha
        \end{align} 
        as long as $ n \geq (\mathcal{G}( \mathcal{F} )^2 + \log(1/\beta)) / \alpha^2 $.
        Thus, for the optimizer $\hat{f}$  from \Cref{alg:erm_perturbed}, we have that 
        \begin{align}
            \mathbb{E} \ell( \hat{f}(x) , y ) \leq \min_{f \in \mathcal{F}}  \mathbb{E} \ell( \hat{f}(x) , y ) + \sup_{f \in \mathcal{F}  } \frac{\eta}{ n \sqrt{m} } \sum_{i=1}^{m} \xi_i f(z_i) + \alpha
        \end{align}  
        From Gaussian concentration for Lipschitz functions \citep[Section 10.5]{boucheron2013concentration}, we have that with probability at least $1 - \beta $, the Gaussian complexity term is bounded by
        \begin{align}
            \mathbb{E} \sup_{f \in \mathcal{F}  } \frac{\eta}{ n \sqrt{m} } \sum_{i=1}^{m} \xi_i f(z_i) + \frac{\eta}{n} \sqrt{\log(1/\beta)} . 
        \end{align}
        Note that this is the same as the Gaussian complexity of the function class $\mathcal{F}$.
    \end{proof}

    Plugging this in (and simplifying) with the choice of $\eta$ from  the privacy guarantee from \Cref{lem:privacy_perturbed}, we get the following result we get that 
    \begin{align}
        n \geq \max \left\{ \frac{\mathcal{G}( \mathcal{F} )^2 + \log(1/\beta) }{\alpha^2} ,  \frac{\mathcal{G}( \mathcal{F} )^2 + \sqrt{ \log(1/\beta) \log(1/\delta) }   } { \alpha \epsilon \rho^2 }    \right\}
    \end{align}
    as required.

}

\section*{Acknowledgments}
AB would like to thank Steven Wu for useful discussions. AS would like to thank the Simons Foundation and the NSF through award DMS-2031883, as well as ARO through award W911NF-21-1-032

\newpage

\bibliographystyle{plainnat}
\bibliography{refs}

\appendix

\crefalias{section}{appendix} %

\clearpage

\section{Related Work}\label{app:related_work}

\paragraph{Oracle-Efficiency in online learning.}
    In order to better develop algorithms for online learning, and motivated by the success of optimization in more classical paradigms,  \cite{kalai2005efficient} introduced the notion of oracle-efficient online learning with the \emph{Follow the Perturbed Leader} family of algorithms.  This approach to online learning serves to stabilize the most na{\"i}ve approach to prediction, i.e. finding the function  that performs best in hindsight and playing this, in order to ensure good performance against adversarial data; subsequent work within this framework \citep{dudik2020oracle,syrgkanis2016efficient,wang2022adaptive,agarwal2019learning,suggala2020online} has significantly generalized the application of these techniques.  One particularly relevant work is that of \citet{hutter2004prediction}, which made a connection between small loss bounds and a stability condition closely resembling differential privacy; this connection was  then made explicit in \cite{abernethy2019online}.  While oracle-efficiency allows for the devlopment of a powerful suite of algorithms,
    unfortunately, \cite{hazan2016computational} showed that in the worst-case, oracle efficiency cannot be achieved for general function classes and adversarial data. 
    In order to circumvent this, a recent line of work has focused on structured instances such as smoothed online learning \citep{haghtalab2020smoothed,oracle-efficient,haghtalab2022smoothed,block2022smoothed,DBLP:journals/corr/abs-2303-04845,DBLP:conf/colt/BlockP23,block2023oracle,pmlr-v247-block24a,block2024smoothed} and has constructed oracle-efficient algorithms for a variety of settings.
    Our work has a conceptually similar goal of circumventing the worst-case hardness of oracle efficiency by focusing on a structured class of functions, $\rho$-separated functions, and constructing oracle-efficient algorithms for online learning.  Unlike these other works, we primarily focus on $\Lstar$ bounds as opposed to minimax regret.  Recently, \citet{wang2022adaptive} also studied $\Lstar$ bounds in the oracle-efficient setting, and we compare our results to theirs in \Cref{sec:main}.

\paragraph{Oracle-Efficiency in differentially private learning.} 
    As in online learning, the using the framework of oracle efficiency to develop practical algorithms for differentially private learning has been an active area of research \citep{DBLP:conf/focs/Neel0W19,Neel0VW20,vietri2020new, gaboardi2014dual,nikolov2013geometry,block2024oracle} and the connections between online learning and differential privacy have been extensively explored \citep{AlonLMM19,bun2020equivalence,abernethy2019online}.  Of particular relevance to our work are the works of \citet{Neel0VW20,block2024oracle}; the former develops oracle-efficient, differentially private learners for function classes with a small separator set (which our results generalize), while the latter uses a similar algorithmic technique as ours to develop oracle-efficient algorithms that makes use of some public data, building upon work in semi-private learning \cite{beimel2014learning,BassilyMA19}

\section{Proof of Lemma \ref{lem:gp_stability_cor}}\label{app:gp_stability}
In this section, we prove \Cref{lem:gp_stability_cor}, the fundamental result in the paper.  The proof proceeds by first proving the technical meat of the lemma (\Cref{lem:gp_stability}), which holds in generality, and then instantiating the result for the case of separated function classes.  The following is a strengthening of the Gaussian process stability result from \citet{block2022smoothed,block2024oracle}.

    \begin{lemma}\label{lem:gp_stability}
        Let $T$ be a set, $K$ a covariance kernel, and suppose that $T$ is separable with respect to the metric induced by $K$.  Suppose that there is some $0 < \kappa \leq K(t,t) \leq 1$ for all $t \in T$, let $\tstar = \argmin \Omega(t)$, and for $\rho, \tau > 0$, define the set 
        \begin{align}\label{eq:good_set_def}
            \cE(\rho, \tau) = \left\{ \text{there exists } s \in T, \text{ s.t. } \frac{K(s,\tstar)}{K(\tstar, \tstar)} \leq 1 - \rho^2 \text{ and } \Omega(s) \leq \Omega(\tstar) + \tau \right\}.
        \end{align}
        Then, for any $t \in T$ it holds that
        \begin{align}\label{eq:gp_stability}
            \pp\left( \tstar = t \right) \leq \left(1 + \frac{2 \tau }{ \eta \kappa^2  \rho^2}  \cdot\left( \ee\left[ \sup_{f \in \cF} \omega(f) \right] + \sqrt{\log\left( \frac 2\delta \right)}\right)\right) \cdot \pp\left( \tstar = t \text{ and } \cE(\rho, \tau)^c \right) + \delta,
        \end{align}
        whenever
        \begin{align}\label{eq:eta_lb}
            \eta \geq \frac{2\tau}{\kappa^2 \rho^2} \left( \ee\left[ \sup_{t \in T} \omega(t) \right] + \sqrt{\log\left( \frac 2\delta \right)} \right).
        \end{align}
    \end{lemma}

\begin{proof}
    We modify the proof of \citet[Theorem 5]{block2024oracle} which in turn is an improvement on \citet{block2022smoothed}.  First, for $\delta > 0$, let
    \begin{align}
        \Phi_\delta = \left\{ \sup_{t \in T} \omega(t) \leq  \ee\left[ \sup_{t \in T} \omega(t) \right] + \sqrt{\log\left( \frac 2\delta \right)}\right\}.
    \end{align}

    By standard Gaussian tailbounds (see, e.g., \citet[Theorem 8.5.5]{vershynin2018high}), it holds that $\pp(\Phi_\delta) \geq 1 - \delta$.  Now, a union bound implies that
    \begin{align}\label{eq:proof1}
        \pp\left( \tstar = t \right) &\leq \pp\left( \tstar = t \text{ and } \cE(\rho, \tau)^c \right) + \pp\left( \tstar = t \text{ and } \cE(\rho, \tau) \text{ and } \Phi_\delta \right) + \delta.
    \end{align}
    We focus on the middle term and show that
    \begin{align}\label{eq:proof2}
        \pp\left( \tstar = t \text{ and } \cE(\rho, \tau) \text{ and } \Phi_\delta \right) \leq \frac{2\tau}{\eta \kappa^2 \rho^2} \left( \ee\left[ \sup_{t \in T} \omega(t) \right] + \sqrt{\log\left( \frac 1\delta \right)} \right) \cdot \pp\left( \tstar = t \right).
    \end{align}
    Given \eqref{eq:proof2}, rearranging \eqref{eq:proof1} and observing that $(1-x)^{-1} \leq 1 + 2x$ for $0 \leq x \leq 1$ concludes the proof of the theorem.

    To establish \eqref{eq:proof2}, we follow \citet{block2024oracle} and introduce for $y \in \rr$ and $s, t \in T$:
    \begin{align}\label{eq:proof3}
        m_{t,y}(s) = m(s) + \frac{K(s,t)}{K(t,t)}(y - m(t)), \qquad a(s) = \frac{\tau}{\rho^2} \cdot \frac{K(s,t)}{K(t,t)}, \qquad \text{and} \qquad b(s) = \frac{\tau}{\rho^2} - a(s).
    \end{align}
    It is immediate that if $K(s,t) \leq (1 - \rho^2) K(t,t)$, then $b(s) \geq \tau$; moreover, $b(s) \geq 0$ for all $s$ by Cauchy-Schwarz and $m_{t, y + \tau/\rho^2}(s) = m_{t,y}(s) + a(s)$ pointwise.  The utility of introducing $m_{t,y}$ is that the distribution of $\Omega(s)$ conditioned on $\Omega(t) = y$ is also a Gaussian process with mean $m_{t,y}$ and covariance $K_t$, independent of $y$.  Thus, we have for all fixed $t \in T$ and $y \in \rr$
    \begin{align}
        \pp&\left( \Phi_\delta, \, \tstar=t \text{, and } \inf_{K(s,t) \leq (1- \rho^2) K(t,t)} \Omega(s) \geq y + \tau | \Omega(t) = y \right) \\
        &\geq \pp\left( \Phi_\delta, \, \tstar=t \text{, and } \inf_{K(s,t) \leq (1- \rho^2) K(t,t)} \Omega(s) - b(s) \geq y | \Omega(t) = y \right) \\
        &=  \pp\left( \Phi_\delta, \, \tstar=t \text{, and } \inf_{K(s,t) \leq (1- \rho^2) K(t,t)} \Omega(s) - b(s) - a(s) + a(s) \geq y | \Omega(t) = y \right) \\
        &=  \pp\left( \Phi_\delta, \, \tstar=t \text{, and } \inf_{K(s,t) \leq (1- \rho^2) K(t,t)} \Omega(s)  + a(s) \geq y + \frac{\tau}{\rho^2} | \Omega(t) = y \right) \\
        &= \pp\left( \Phi_\delta, \, \tstar=t \text{, and } \inf_{K(s,t) \leq (1- \rho^2) K(t,t)} \Omega(s) \geq y + \frac{\tau}{\rho^2} | \Omega(t) = y + \frac{\tau}{\rho^2} \right).
    \end{align}
    Letting
    \begin{align}
        q_t(y) = (2 \pi K(t,t,)^{-1/2}) \cdot \exp\left( -\frac{(y - m(t))^2}{2 \eta^2 K(t,t)} \right)
    \end{align}
    denote the density of $\Omega(t)$, we have
    \begin{align}
        \pp&\left( \tstar = t \text{ and } \cE(\rho, \tau) \text{ and } \Phi_\delta \right) \\
        &= \pp\left( \tstar = t  \text{ and } \Phi_\delta \right)  - \int_{-\infty}^\infty \pp\left( \Phi_\delta, \, \tstar=t \text{, and } \inf_{K(s,t) \leq (1- \rho^2) K(t,t)} \Omega(s) \geq y + \tau | \Omega(t) = y \right)q_t(y) d y \\
        &\leq \pp\left( \tstar = t  \text{ and } \Phi_\delta \right)  - \int_{-\infty}^\infty \pp\left( \Phi_\delta, \, \tstar=t \text{, and } \inf_{K(s,t) \leq (1- \rho^2) K(t,t)} \Omega(s) \geq y + \frac{\tau}{\rho^2} | \Omega(t) = y + \frac{\tau}{\rho^2} \right) q_t(y)d y \\
        &= \int_{-\infty}^\infty \left( q_t(y) - q_t\left( y - \frac{\tau}{\rho^2} \right) \right) \pp\left( \Phi_\delta, \, \tstar=t \text{, and } \inf_{K(s,t) \leq (1- \rho^2) K(t,t)} \Omega(s) \geq y  | \Omega(t) = y  \right) d y \\
        &= \int_{-\infty}^\infty \left( q_t(y) - q_t\left( y - \frac{\tau}{\rho^2} \right) \right) \pp\left( \Phi_\delta \text{ and } \tstar=t  | \Omega(t) = y  \right) d y.
    \end{align}
    Noting now that
    \begin{align}
        q_t(y) - q_t\left( y - \frac{\tau}{\rho^2} \right) \leq \frac{\tau q_t(y)}{ \eta^2 \kappa^2  \rho^2} (y - m(t))
    \end{align}
    by the proof of \citet[Theorem 5]{block2024oracle}, we have that
    \begin{align}
        \pp\left( \tstar = t \text{ and } \cE(\rho, \tau) \text{ and } \Phi_\delta \right) &\leq \frac{\tau }{ \eta^2 \kappa^2  \rho^2} \int_{-\infty}^\infty (y - m(t)) \pp\left( \Phi_\delta \text{ and } \tstar=t  | \Omega(t) = y  \right) q_t(y) d y \\
        &= \frac{\tau }{ \eta^2 \kappa^2  \rho^2} \cdot  \ee\left[ \eta  \cdot \omega(t) \cdot \bbI\left[ \Phi_\delta \text{ and } \tstar = t \right] \right] \\
        &= \frac{\tau }{ \eta \kappa^2  \rho^2} \cdot \ee\left[ \omega(\tstar) \bbI[\Phi_\delta] | \tstar = t \right] \cdot \pp\left( \tstar = t \right).
    \end{align}

    Now, noting that, by definition, 
    \begin{align}
        \ee\left[ \omega(\tstar) \bbI[\Phi_\delta] | \tstar = t \right]  \leq \ee\left[ \sup_{t \in T} \omega(t) \right] + \sqrt{\log\left( \frac 2\delta \right)},
    \end{align}
    we conclude the proof.
\end{proof}

We now instantiate \Cref{lem:gp_stability} for the case of separated index sets and bounded perturbations following \citet[Corollary 1]{block2024oracle}
\begin{theorem}\label{thm:separated_privacy}
    Let $T, \omega, \Omega, K, m$ be as in \Cref{lem:gp_stability} and suppose that $T$ is a subset of a real vector space, $K$ is bilinear and let $\norm{t} = \sqrt{K(t,t)}$, extended through bilinearity to the span of $T$.  Further suppose that $T$ is $\rho$-separated in $\norm{\cdot}$ in the sense that $\norm{t - t'} \geq \rho$ for all $t \neq t' \in T$.  Let $m' : T \to \rr$ denote a mean function such that $\sup_{t \in T} \abs{m(t) - m(t')} \leq \tau$ and let $\Omega'$ denote the associated shifted Gaussian process.  Letting $\tstarp = \argmin \Omega'(t)$, if
    \begin{align}\label{eq:separated_privacy_eta_lb}
        \eta \geq \frac{32\tau}{\rho^2} \left( \ee\left[ \sup_{t \in T} \omega(t) \right] + \sqrt{\log\left( \frac 2\delta \right)} \right),
    \end{align}
    it holds for any $t \in T$ that
    \begin{align}
        \pp\left( \tstar = t \right) \leq \left( 1 + \frac{32\tau}{\eta \rho^2} \cdot \left( \ee\left[ \sup_{f \in \cF} \omega(f) \right] + \sqrt{2 \log\left( \frac 2\delta \right)} \right) \right) \cdot \pp\left( \tstarp = t \right) + \delta
    \end{align}
\end{theorem}
\begin{proof}
    Note that unlike \Cref{lem:gp_stability}, we do not requite a lower bound of $\kappa \leq K(t,t)$.  To handle this case, we introduce the kernel $\Ktil: T \times T \to \rr$ defined such that
    \begin{align}
        \Ktil(s,t) = \frac 12 \left( K(s,t) + \kappa^2 \right)
    \end{align}
    for some $1 \geq \kappa > 0$.  Observe that $\Ktil = (K + \kappa^2 \cdot \mathbf{1} \mathbf{1}^\top)/2 \succeq K / 2$, and, because $\mathbf{1} \mathbf{1}^\top$ is rank 1, its addition does not affect whether or not $K$ is trace class; thus $\Ktil$ is a valid covariance kernel if $K$ is.  Furthermore, it is clear that $\Ktil(t, t) \geq \kappa^2 / 2$ for all $t \in T$ and that $\Ktil$ is the covariance kernel associated to the process $\Omegatil(t) = \Omega(t) + \kappa \cdot \xi$, where $\xi$ is an independent standard normal random variable.  Thus for any $t \in T$, it holds that
    \begin{align}\label{eq:same_distribution}
        \pp\left( \ttilstar = t \right) = \ee_\xi \left[ \pp\left( \ttilstar = t | \xi \right) \right] = \ee_\xi\left[ \pp\left( \tstar = t \right) \right] = \pp\left( \tstar = t \right).
    \end{align}
    Thus $\tstar$ and $\ttilstar$ have the same distribution.  Let $\Omegatil' = \Omegatil + m' - m$, i.e., the process with mean $m'$ and covariance kernel $\Ktil$ and observe the above holds for $\Omegatil'$ and its minimizer $\ttilstarp$ as well.  Furthermore, if we compute the norm induced by $\Omegatil$, we see that
    \begin{align}
        \Ktil(t,t) + \Ktil(s,s) - 2 \Ktil(s,t) &= \frac{K(t,t) + \kappa^2}{2} + \frac{K(s,s) + \kappa^2}{2} - K(s,t) - \kappa^2 \\
        &= \frac{K(t,t) + K(s,s) - 2 K(s,t)}{2} \\
        &= \frac{\norm{s -t}^2}{2}.
    \end{align}
    Now, applying \Cref{lem:gp_stability} to $\Omegatil$, we see that for sufficiently large $\eta$, it holds that
    \begin{align}
        \pp\left( \ttilstar = t \right) \leq \left( 1 + \frac{16 \tau}{\eta \kappa^2 \rho^2} \cdot \left( \ee\left[ \sup_{f \in \cF} \omega(f) \right] + \sqrt{2 \log\left( \frac 2\delta \right)} \right) \right) \cdot \pp\left( \ttilstar = t \text{ and } \cE^c(\rho/2, 2 \cdot\tau) \right) + \delta,
    \end{align}
    where we note that the expected supremum of $\omega$ is unaffected by the addition of $\kappa \cdot \xi$; a similar statement holds for $\Omegatil'$.  Furthermore, by construction,  
    \begin{align}
        \Omegatil(\ttilstarp) = \Omegatil'(\ttilstarp) + m'(\ttilstarp) - m(\ttilstarp) \leq \Omegatil'(\ttilstar) +  m'(\ttilstarp) - m(\ttilstarp) \leq \Omega(\ttilstar) + 2\tau.
    \end{align}
    To conclude, we note that if $\ttilstar \neq \ttilstarp$, then $\norm{\ttilstar - \ttilstarp}^2 \geq \rho^2$, and, letting $M = \max\left( \Ktil(\ttilstar, \ttilstar), \Ktil(\ttilstarp, \ttilstarp) \right) \leq 1$, we see that
    \begin{align}
        \norm{\ttilstar - \ttilstarp}^2 &= 2 \cdot \left( \Ktil(\ttilstar, \ttilstar) + \Ktil(\ttilstarp, \ttilstarp) - K(\ttilstar, \ttilstarp) \right) \\
        &\leq 4 M \left( 1 - \frac{K(\ttilstar, \ttilstarp)}{M} \right).
    \end{align}
    Rearranging, we see that
    \begin{align}
        \min\left( \frac{\Ktil(\ttilstar, \ttilstarp)}{\Ktil(\ttilstar, \ttilstar)}, \frac{\Ktil(\ttilstar, \ttilstarp)}{\Ktil(\ttilstarp, \ttilstarp)} \right) \leq 1 - \frac{\rho^2}{4M} \leq 1 - \frac{\rho^2}{4}.
    \end{align}
    Thus we see that $\left\{ \ttilstarp \neq \ttilstar \right\} \subset \cE\left( \nicefrac \rho2, 2\tau \right)$.  Recalling from \eqref{eq:same_distribution} that $\tstar$ and $\ttilstar$ have the same distribution, we conclude the proof by setting $\kappa =1$.

\end{proof}
Finally, the proof of \Cref{lem:gp_stability_cor} follows by specializing to the setting where $\cF$ is a function class mapping $\cX \to \rr$ and $\omega$ is the canonical Gaussian process associated to this space.

\section{Proof of Theorem \ref{thm:lstar_online}}\label{app:online_learning}

In this section, we prove \Cref{thm:lstar_online}.  
We rely on the `Be-The-Leader' Lemma introduced in \citet{kalai2005efficient} and discussed in \citet{cesa2006prediction}.  In particular, we invoke the following version of the lemma:
\begin{lemma}[Lemma 31 from \citet{block2022smoothed}]\label{lem:btl}
    Let $x_t, y_t$ be a possibly adaptively chosen sequence of contexts and labels, let $\ell$ be a loss function, denote $\ell_t(f) = \ell(f(x_t), y_t)$, and let 
    \begin{align}
        L_t(f) = \sum_{s = 1}^t \ell_t(f).
    \end{align}
    Let $R_t: \cF \to \rr$ be a sequence of identically distributed functionals on $\cF$ and let
    \begin{align}
        f_t \in \argmin_{f \in \cF} L_{t-1}(f) + R_t(f).
    \end{align}
    Then it holds for any learner playing $f_t$ at time $t$ that
    \begin{align}
        \ee\left[ \reg_T \right] \leq \ee\left[ \sup_{f \in \cF} R_1(f) \right] - \ee\left[ \inf_{f \in \cF} R_1(f) \right] + \sum_{t = 1}^T \ee\left[ \ell_t(f_t) - \ell_t(f_{t+1}) \right].
    \end{align}
    Furthermore,
    \begin{align}
        \ee\left[ \sum_{t = 1}^T \ell_t(f_{t+1}) \right] \leq \Lstar + \ee\left[ \sup_{f \in \cF} R_1(f) \right] - \ee\left[ \inf_{f \in \cF} R_1(f) \right]
    \end{align}
\end{lemma}
\Cref{lem:btl} allows us to control the regret of FTPL style algorithms by bounding two terms: a `bias' term measuring the size of the perturbation and a `variance' term that controls the stability of the predictions enforced by the perturbation itself.  In our case, we invoke the lemma with $
\omega_t$ from \eqref{eq:ftpl_gaussian} as the perturbation and observe that, because $\omega_t(f)$ is symmetric, we can collaps the difference between expected supremum and infimum into a single term.  We now apply \Cref{lem:gp_stability_cor} to control the stability term.
\begin{proof}[\pfref{thm:lstar_online}]
    By \Cref{lem:btl}, it holds that
    \begin{align}
        \ee\left[ \reg_T \right] &\leq 2 \eta \cdot \ee\left[ \sup_{f \in \cF} \omega_1(\cF) \right] + \sum_{t = 1}^T \ee\left[ \ell_t(f_t) - \ell_t(f_{t+1}) \right] \\
        &= 2 \eta \cdot \ee\left[ \sup_{f \in \cF} \omega_1(\cF) \right] + \sum_{t = 1}^T \sum_{f \in \cF} \ell_t(f)  \left( \pp(f_t = f) - \pp(f_{t+1} = f) \right). \label{eq:online_learning_proof1}
    \end{align}
    For a fixed $t$ and fixed $f$, note that because $\abs{L_{t-1}(f) - L_t(f)} \leq B$ for all $f$, we may apply \Cref{thm:separated_privacy}
    \begin{align}
        \pp(f_t = f) - \pp(f_{t+1} = f) &\leq  \delta + \left(1 + \frac{64 B }{ \eta   \rho^2}  \cdot\left( \ee\left[ \sup_{f \in \cF} \omega_t(f) \right] + \sqrt{\log\left( \frac 2\delta \right)}\right)\right) \cdot \pp\left( f_{t+1} = f \right) - \pp(f_{t+1} = f) \\
        &= \delta + \frac{64 B }{ \eta  \rho^2}  \cdot\left( \ee\left[ \sup_{f \in \cF} \omega_t(f) \right] + \sqrt{\log\left( \frac 2\delta \right)}\right) \cdot \pp\left( f_{t+1} = f \right),
    \end{align}
    as long as
    \begin{align}\label{eq:online_learning_proof_etalb}
        \eta \geq \frac{64 B }{   \rho^2}  \cdot\left( \ee\left[ \sup_{f \in \cF} \omega_t(f) \right] + \sqrt{\log\left( \frac 2\delta \right)}\right)
    \end{align}
    Setting $\delta = \nicefrac{1}{T \abs{\cF}}$, we see that
    \begin{align}\label{eq:online_learning_proof2}
        \pp(f_t = f) - \pp(f_{t+1} = f) \leq \frac{1}{T \abs{\cF}} + \frac{64 B}{\eta\rho^2} \cdot\sqrt{\log\left( 2 T |\cF| \right)} \cdot \pp\left( f_{t+1} = f \right)
    \end{align}
    Plugging \eqref{eq:online_learning_proof2} into \eqref{eq:online_learning_proof1} we get
    \begin{align}
        \ee\left[ \reg_T \right] \leq 2 \eta \cdot \sqrt{\log \abs{\cF}} + 1 + \frac{64 B}{\eta  \rho^2} \cdot\sqrt{\log\left( 2 T |\cF| \right)} \cdot \left( \Lstar + 2 \cdot \ee\left[ \sup_{f \in \cF} \omega_1(f) \right] \right)
    \end{align}
    whenever $\eta$ satisfies \eqref{eq:online_learning_proof_etalb}.  Now, if
    \begin{align}
        \frac{2 \cdot \sqrt{B (\Lstar + \log\left( \abs{\cF} \right))}}{\rho} \geq \frac{64 B }{   \rho^2}  \cdot\sqrt{\log\left(  2T \abs{\cF} \right)},
    \end{align}
    then we may set $\eta$ to the left-hand side above; otherwise, we minimize $\eta$ subject to \eqref{eq:online_learning_proof_etalb} to get the desired result.
\end{proof}

\section{Miscellanious Proofs}\label{app:miscellany}

\subsection{Proof of Proposition \ref{prop:hadamard_gap}}
\label{sec:hadamard_gap}

In order to compare our result to that of \citet{wang2022adaptive}, we consider the class of all linear functions over $\mathbb{F}_2^n$ which we refer to as the Hadamard class $\mathcal{F}_{\mathrm{Had}}$, which we now define. 

\begin{definition}[Hadamard Class]
    Let the domain be $ \mathbb{F}_2^n $ i.e. the set of all $n$-bit strings with addition modulo 2.
    The Hadamard class $\mathcal{F}_{\mathrm{Had}}$ is the set of all linear functions $f_y: \mathbb{F}_2^n \to \mathbb{F}_2 $ given by $f_y(x) = \inprod{x}{y}$ for $y \in \mathbb{F}_2^n$.
\end{definition}

By standard abuse of notation, we will associate the output $ \mathbb{F}_2 $ with $ \left\{ 0,1 \right\} \subset \mathbb{R} $.   
We first note that this class is $\rho$-separated by the uniform distribution on $\mathbb{F}_2^n$ with $\rho = 1/\sqrt{2}$. 
This follows from the fact that for any $y \neq 0 $, there exists $\mathbb{P}_{x \sim \mathbb{F}_2^n } ( \inprod{x}{y} = 0 ) = 1/2  $.
While one can show with a similar argument $\gamma$-approximability parameter of this class is $ \gamma = O(1) $, note that lead term in the regret bound \Cref{eq:wang_regret} survives even if we set $\gamma = 0$.
Thus, evaluating the regret bound of \citet{wang2022adaptive} for this class, we find that using their algorithm and plugging in \eqref{eq:wang_regret}, regret is bounded as
\begin{align}
    \ee\left[ \reg_T \right]\lesssim \left( n + \sqrt{n 2^n}\right) \cdot \left(1 + \sqrt{\Lstar} \right).
\end{align}  
On the other hand, if we use \Cref{alg:ftpl}, we see by \Cref{thm:lstar_online} that
\begin{align}\label{eq:hadamard_guassian}
    \ee\left[ \reg_T \right] \lesssim 1 + n \log(T) +  \sqrt{n \log(T) \Lstar},
\end{align}
which is an exponential improvement in $n$ in regret.  While this separation is extreme, the size of the separating set is $2^n$,  where $n$ can be considered the natural parametrization of the size of the problem, and thus the oracle complexity of both algorithms is exponential. 

We can modify the construction, however, to preserve a separation while keeping the oracle complexity polynomial.  Indeed, note that if we  consider the set $ \tilde{\mathcal{X}} = \left\{ e_i  \right\} $ where $e_i$ are the standard basis vectors of $ \mathbb{F}_2^n $, then for the uniform distribution over $ \tilde{\mathcal{X}} $, the class $ \mathcal{F}_{\mathrm{Had}} $ is $ \rho $-separated with $ \rho = 1/ \sqrt{2} $ for the same reason as before.
In this case, the regret bound of \citet{wang2022adaptive} becomes
\begin{align}
    \ee\left[ \reg_T \right] \lesssim   \left( 1 + n \right) \cdot \left(1 +  \sqrt{\Lstar} \right), 
\end{align}
while the bound from \Cref{thm:lstar_online} remains as in \eqref{eq:hadamard_guassian}.  This translates to a $\sqrt{n}$ improvement in the leading term of the regret bound of our algorithm over that of \citet{wang2022adaptive}.

\subsection{Proof of \Cref{prop:singular_value}}\label{app:singular_value}

This follows from a standard argument in linear algebra.  Indeed, abusing notation, let $f, f'$ be a pair of functions in $\cF$ projected onto $\cZ$ and note that there exist basis vectors $u, u' \in \rr^{m}$ such that $f -f' = A (u - u')$.  By definition of the minimal singular value, it holds that
\begin{align}
    \norm{f -f'}_2 = \norm{A (u - u')}_2 \geq \sigma \cdot \norm{u - u'}_2 = \sigma \cdot \sqrt{2}.
\end{align} 
The result follows by renormalizing.

\subsection{Proof of \Cref{lem:separating_measure_sample}} 
\label{sec:separator_proofs}

    In this section we state and prove a more general version of \Cref{lem:separating_measure_sample}, which applies to arbitrary real-valued function classes.  We first recall two more notions of function class complexity: fat-shattering dimension \citep{bartlett1994fat} and covering numbers.
    \begin{definition}
        Let $\cF: \cX \to \rr$ be a function class.  We say that $\cF$ is shattered at scale $\alpha > 0$ by points $z_1, \dots, z_m$ if there exists some witness $s \in \rr^m$ such that for all $\epsilon \in \left\{ \pm 1 \right\}^m$ there is an $f_\epsilon \in \cF$ satisfying $\epsilon_i (f_\epsilon(z_i) - s_i) \geq \alpha / 2$.  The fat-shattering dimension at scale $\alpha$, $\fat_\alpha(\cF)$ is the maximal number of points that can be shattered at scale $\alpha$.
    \end{definition}

    \begin{definition}
        Let $\cF: \cX \to \rr$ be a function class.  The $\epsilon$-covering number of $\cF$ with respect to a measure $\mu$, denoted by $N(\epsilon, \cF, \mu)$ is the smallest cardinality of $\mathcal{F'} \subseteq \mathcal{F} $ such that for all $f \in \mathcal{F}$, there exists a $f' \in \mathcal{F}' $ such that $ \norm{ f - f' }_{L^2(\mu)} <  \epsilon $.
    \end{definition}

    The following fundamental result controls covering numbers by fat-shattering dimension.
    \begin{theorem}[\citet{mendelson2003entropy}]\label{thm:mendelson}
        Let $\cF: \cX \to \rr$ be a function class.  There exist universal constants $c, C > 0$ such that for all $\epsilon$ and all $\mu$, it holds that
        \begin{align}
            N(\epsilon, \cF, \mu) \leq \left( \frac 2\epsilon \right)^{C\cdot \fat_{c \epsilon}(\cF)}.
        \end{align}
    \end{theorem}
    In the special case where $\cF$ is binary-valued, $\fat_0(\cF)$ becomes the VC dimension and the classical Dudley extraction result \citep{dudley1969speed} follows.  We have the following result.

    \begin{lemma}
        Let $\cF: \cX \to \rr$ be a function class and suppose that $\mu$ is a $\rho$-separating measure.  Then there exist constants $c, C > 0$ such that
        \begin{align}
            \abs{\cF} \leq \left( \frac 2 \rho \right)^{C \cdot \fat_{c \rho}(\cF)}.
        \end{align}
        Moreover, there exists a $\rho$-separating set of size $m = \bigO\left( \fat_{c \rho}(\cF) / \rho^2 \right)$.  In particular, if $\cF$ is a binary-valued function class of VC dimension $d$, then $\abs{\cF} \leq \left( \frac{10} \rho \right)^{d}$ and there exists a separator set of size $m = \bigO\left( d / \rho^2 \right)$.
    \end{lemma}
    \begin{proof}
        We first note that if $\cF$ has a $\rho$-separating measure $\mu$, then $\abs{\cF} \leq N\left( \rho, \cF, \mu \right)$.  This is because if $\mathcal{F}'$ is a $\rho$ covering of $\mathcal{F}$, we must have $\mathcal{F} = \mathcal{F}'$. 
        To see this note that, from $\rho$ separation, the set $ \left\{ f \in  \mathcal{F} : \norm{f - f'}_{\mu} < \rho    \right\} = \{ f'\}    $. 
        Since, we have $\mathcal{F}'$ is a $\rho$-covering, by singletons, we must have $\mathcal{F} = \mathcal{F}'$.   The first statement follows by applying \Cref{thm:mendelson}.

        On the other hand, by \Cref{lem:separator_set_2} and the probabilistic method, it holds that there exists some set of size $m = 2 \log( N( \rho , \mathcal{F} , \mu )  ) / \rho^{2}$ that $\rho$-separates $\cF$.  Again applying \Cref{thm:mendelson} gives the second statement.  The binary-valued $\cF$ case is proved similarly.
    \end{proof}

\begin{lemma}
    \label{lem:separator_set_2}
    Let $\cF$ be a function class such that $\mu$ is a $\rho$-separating measure for $\cF$.
    Then, with probability $1 - \delta$, a sample $S = \left\{ x_i \right\}$  of size $m = ( \log ( N( \rho , \mathcal{F} , \mu )  ) + \log(1/ \delta)) \rho^{-2}$ 
    \begin{align}
        \norm{f - g}_m^2 = \frac{1}{m} \sum_i (f(x_i) - g(x_i))^2  \geq \frac{\rho}{2}. 
    \end{align} 
    Further, this set is a separator set for $\cF$.
    \end{lemma}
    \begin{proof}
        The proof follows from standard uniform convergence arguments for finite classes \citep{shai}.
    \end{proof}
    
    We note here that in the presence of $\rho$-separation for $\mu$, we get uniform convergence in terms of the covering number for $\mu$ as opposed to the standard result in terms of the uniform convering numbers $ \sup_{\nu} N( \rho , \mathcal{F} , \nu )  $.

\section{Proof of \Cref{cor:gamma_approx}}
\label{sec:approximability}

    Supposing that the empirical measure on $m$ points $\cZ = \left\{ Z_1, \dots, Z_m \right\}$ provides $\rho$-separation (\Cref{def:rho_separating}), we see that this latter condition is equivalent to
\begin{align}\label{eq:rho_separation_dual}
    \rho \leq \inf_{f \neq f' \in \cF} \norm{f - f'}_m = \inf_{f \neq f' \in \cF} \sup_{\substack{s \in \rr^m \\ \norm{s}_2 \leq 1}} \inprod{(f - f')|_\cZ}{\nicefrac{s}{\sqrt{m}}},
\end{align}
    Expanding \Cref{def:approximable}, we get
\begin{align}\label{eq:gamma_approx}
    \frac{1}{\gamma} \leq \inf_{f \in \cF} \sup_{\substack{s \in \rr^m \\ \norm{s}_1 \leq 1}} \inf_{f' \neq f \in \cF} \inprod{\Gamma_f - \Gamma_{f'}}{s}.
\end{align}
It is now apparent that a weaker condition than $\gamma$-approximability, the requirement that $\Gamma_f$ are separated in $\norm{\cdot}_\infty$, implies $\nicefrac{1}{\gamma \sqrt{m}}$ separation, leading to the required corollary.

\end{document}